%% file: main.tex
\documentclass[sigconf,nonacm]{aamas} 

\usepackage{balance} 
\input{auxiliary/packages}
\input{auxiliary/macros}
\newtheorem{proof}{Proof}
\newtheorem{definition}{Definition}
\newtheorem{remark}{Remark}
\newtheorem{theorem}{Theorem}

\newtheorem{lemma}{Lemma}
\newtheorem{corollary}{Corollary}

\setcopyright{ifaamas}
\acmConference[AAMAS '26]{Proc.\@ of the 25th International Conference
on Autonomous Agents and Multiagent Systems (AAMAS 2026)}{May 25 -- 29, 2026}
{Paphos, Cyprus}{C.~Amato, L.~Dennis, V.~Mascardi, J.~Thangarajah (eds.)}
\copyrightyear{2026}
\acmYear{2026}
\acmDOI{}
\acmPrice{}
\acmISBN{}

\acmSubmissionID{223}

\title[]{
Robust Probabilistic Shielding
\\for Safe Offline Reinforcement Learning
}

\author{Maris~F.~L.~Galesloot}
\affiliation{
  \institution{Radboud University}
  \city{Nijmegen}
  \country{The Netherlands}}
\email{maris.galesloot@ru.nl}

\author{Thomas Rhemrev}
\affiliation{
  \institution{Radboud University}
  \city{Nijmegen}
  \country{The Netherlands}}
\email{thomas.rhemrev@ru.nl}

\author{Nils Jansen}
\affiliation{
  \institution{\mbox{Ruhr University\,\&\,Radboud University}}
  \city{Bochum}
  \country{Germany}}
\email{n.jansen@rub.de}

\begin{abstract}
\input{sections/00-abstract}
\end{abstract}


\newcommand{\BibTeX}{\rm B\kern-.05em{\sc i\kern-.025em b}\kern-.08em\TeX}

\makeatletter
\newcommand{\settotalpages}[1]{%
  \@namedef{r@TotPages}{{#1}{}}%
}
\makeatother

\begin{document}


\pagestyle{fancy}
\fancyhead{}


\maketitle 


\input{sections/01-introduction}

\input{sections/02-preliminaries}

\input{sections/03-shielding}
\input{sections/04-experiments}
\input{sections/05_discussion_conclusion}



\begin{acks}
This work has been partially funded by the ERC Starting Grant DEUCE (101077178) and the ORLEANS project from the Interdisciplinary Research Platform of Radboud University.
\end{acks}



\bibliographystyle{ACM-Reference-Format} 
\bibliography{references}


\appendix
\input{sections/A1-complete_graphs}

\end{document}

%% file: auxiliary/packages.tex
\usepackage{comment}

\usepackage{amsmath}

\IfPackageLoadedTF{amssymb}{}{\usepackage{amssymb}}

\usepackage{amsthm}
\usepackage{xargs}
\usepackage{tikz}
\usepackage{nicefrac}
\usepackage[textwidth=33mm,
	backgroundcolor=white,
    tickmarkheight=2pt,
	textsize=tiny]{todonotes}
\usepackage{xspace}
\usepackage{booktabs}
\usepackage{enumitem}
\usepackage[capitalize]{cleveref}
\usepackage{subcaption}
\usepackage{mdframed}
\IfPackageLoadedTF{natbib}{}{\usepackage[numbers]{natbib}}
\usetikzlibrary{positioning, shapes.geometric, arrows.meta, fit}

%% file: auxiliary/macros.tex
\newcommand{\tool}[1]{{\textsc{#1}}}
\newcommand{\storm}{\tool{Storm}\xspace}
\newcommand{\prism}{\tool{Prism}\xspace}

\DeclareMathOperator*{\argmax}{argmax}

\DeclareMathOperator{\supp}{supp}

\theoremstyle{definition} 

\newtheorem{assumption}{Assumption}

\newcommand{\E}{\mathbb{E}}

\newcommand{\Rmax}{R_{\text{max}}}
\newcommand{\T}[0]{{T}}
\newcommand{\R}{R}
\newcommand{\distr}[1]{\Delta({#1})}
\newcommand{\mdp}{{{M}}}
\newcommand{\imdp}{\mdp_{\intervals}}

\newcommand{\intervals}{\mathbb{I}}

\newcommand{\safestates}{S_T}
\newcommand{\unsafestates}{S_U}
\newcommand{\dir}{\text{Dir}}
\newcommand{\maxprob}{\mathbb{P}^{\hat{\imdp}, \Pi,\varphi}_{\text{maxmin}}}

\newcommand{\dataset}{\mathcal{D}}
\newcommand{\boots}{\mathcal{B}}
\newcommand{\nwedge}{N_{\wedge}}
\newcommand{\var}{\operatorname{Var}}
\newcommand{\shield}{%
    \tikz[baseline]{\draw (0,1.0ex) -- (0,0.0ex) arc [radius=0.5ex, start angle=-180, end angle=0] -- (1.0ex,1.0ex) -- cycle;}%
}
\newcommand{\threshold}{{\theta}}
\newcommand{\buffer}{\kappa}
\newcommand{\errorrate}{\eta}
\newcommand{\probmargin}{\xi}
\newcommand{\uncertaintyweight}{\nu}
\newcommand{\confspi}{\delta_{\rho}}
\newcommand{\confshield}{{\delta_{\intervals}}}
\newcommand{\confimdp}{{\delta_{\intervals}}}

\newcommand{\shieldspibb}{SPIBB$_{\shield}$\xspace}

\newcommand{\policy}{\pi}
\newcommand{\baseline}{\policy^b}
\newcommand{\heuristic}{\policy^h}
\newcommand{\improved}{\policy^{\iota}}
\newcommand{\shieldedbaseline}{\shielded{\baseline}}
\newcommand{\shieldedimproved}{\shielded{\improved}}
\newcommand{\shielded}[1]{#1_{\shield}}

\newcommand{\mlmdp}{MLE-MDP\xspace}
\newcommand{\error}{e_{\dataset}}
\newcommand{\admissibleset}{\Xi^{\error}_{\hat\mdp}}

\newcommand{\powerset}{\mathcal{P}}

\newcommand{\performance}[1]{\rho(#1)}

\newcommand{\rrV}{{\mathbb{V}}_{\hat{\imdp}}}
\newcommand{\rV}{\mathbb{V}_{\mdp^*}}
\newcommand{\rrQ}{{\mathbb{Q}}_{\hat{\imdp}}}
\newcommand{\rQ}{\mathbb{Q}_{\mdp^*}}

\newlist{questionenum}{enumerate}{1}
\setlist[questionenum]{label=\textbf{(Q\arabic*)}, ref=(Q\arabic*)}
\Crefname{question}{Question}{Questions}
\crefname{question}{question}{questions}
\crefalias{questionenumi}{question}

%% file: sections/00-abstract.tex
In \emph{offline} reinforcement learning (RL), we learn policies from fixed datasets without environment interaction.
The major challenges are to provide guarantees on the (1) performance and (2) safety of the resulting policy.
A technique called \emph{safe policy improvement}~(SPI) provides a \emph{performance guarantee}: with high probability, the new policy outperforms a given \emph{baseline policy}, which is assumed to be safe.
Orthogonally, in the context of safe RL, a \emph{shield} provides a \emph{safety guarantee} by restricting the action space to those actions that are provably safe with respect to a given \emph{safety-relevant model}.
We integrate these paradigms by extending shielding to offline RL, relying solely on the available dataset and knowledge of safe and unsafe states.
Then, we shield the policy improvement steps, guaranteeing, with high probability, a safe policy.
Experimental results demonstrate that shielded SPI outperforms its unshielded counterpart, improving both average and worst-case performance, particularly in low-data regimes.

%% file: sections/01-introduction.tex
\section{Introduction}
In an \emph{online} reinforcement learning (RL) setting, an agent learns to make decisions by interacting with an environment, modeled as a Markov decision process (MDP)~\citep{Bellman1957AMD,DBLP:books/wi/Puterman94}, to maximize cumulative discounted rewards. 
In settings where interaction with the environment is infeasible, e.g., when these interactions are costly, dangerous, or otherwise impractical for real-world applications, \emph{offline} RL~\citep{DBLP:journals/corr/abs-2005-01643} (also referred to as \emph{batch} RL~\citep{DBLP:conf/ecml/ErnstGW03}) provides a feasible alternative.
In offline RL, the agent learns a policy solely by analyzing a dataset of transitions in the MDP that was previously collected by a \emph{baseline policy}. 
To allow for deployment in safety-critical scenarios, a challenge of offline RL is to provide guarantees on the (1)~\emph{performance} and (2)~\emph{safety} of the resulting policy.

\begin{figure}[t]
    \centering
    \includegraphics[width=0.9\linewidth]{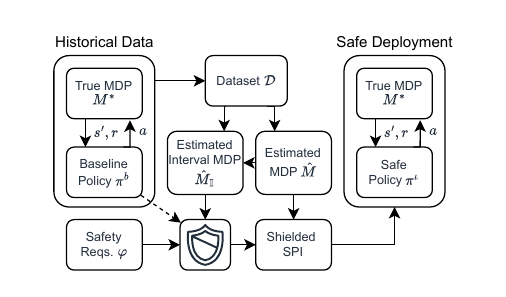}
    \caption{The shielded SPI methodology for safe offline RL.}
    \label{fig:placeholder}
\end{figure}

For guaranteeing performance, a promising direction is so-called \emph{safe policy improvement}~(SPI)~\citep{DBLP:conf/icml/PirottaRPC13,DBLP:conf/nips/GhavamzadehPC16,DBLP:conf/icaart/SchollDOU22a}. 
In offline RL, the SPI problem is to find a policy based on a given dataset that improves over a given (or estimated~\citep{DBLP:conf/atal/SimaoLC20}) baseline policy with high probability~\citep{DBLP:conf/icml/LarocheTC19}.
Unfortunately, the number of times each state-action pair must occur in the dataset is prohibitively large to guarantee a tight improvement with high probability.
A well-known algorithm that addresses the SPI problem in offline RL is SPI \emph{with baseline bootstrapping}~(SPIBB)~\citep{DBLP:conf/icml/LarocheTC19}.
In SPIBB, the improvement requirement is relaxed by falling back to the baseline policy for parts of the MDP with low coverage in the dataset.
SPIBB rests on the critical assumption that the fallback baseline policy must be safe, since it is not always improved. 
Furthermore, the notion of ``safety'' in SPIBB merely provides approximate guarantees of improvement over an assumed-to-be-safe policy.
No lower bound or guarantees on the actual safety of the resulting policies are guaranteed.
Crucially, in many real-world scenarios, there are \emph{safety requirements} to which the policies must adhere, and we must have a \emph{safety guarantee} that these requirements are satisfied.

In \emph{safe} RL, \emph{shielding}~\citep{DBLP:conf/concur/0001KJSB20} is a technique to ensure that an agent does not take actions that would lead to undesirable or unsafe states, providing a \emph{safety guarantee}. 
For instance, a \emph{reach-avoid} specification governs the probability of reaching safe states while avoiding unsafe states~\citep{model_checking_bible}.
Then, a \emph{probabilistic shield} restricts the actions that exceed a certain threshold probability of violating such a specification~\citep{DBLP:conf/concur/0001KJSB20,court_probabilistic_2025}.
Shielding is commonly applied in online RL.
The major drawback is that the computation of a shield requires the availability of a model in the form of a \emph{safety-relevant} MDP, encompassing all the necessary information for safe behavior~\citep{DBLP:conf/concur/0001KJSB20,court_probabilistic_2025,DBLP:conf/aaai/AlshiekhBEKNT18}.
All safety guarantees of shielding rely on this safety-relevant MDP, which is generally unrealistic to assume.
To the best of our knowledge, probabilistic shielding has not been studied in an offline RL setting where only a dataset of limited size is available.

\paragraph{Contributions.}
In this paper, we introduce probabilistic shields to offline RL.
Specifically, we integrate probabilistic shielding into SPI.
We compute a shield from a given dataset, assuming knowledge of safe and unsafe states.
To account for the potential imperfection of the available dataset, we base the shield computation on the construction of a so-called \emph{robust MDP}~\citep{DBLP:journals/mor/WiesemannKR13}.
These models capture data uncertainty by featuring imprecise transition probabilities via \emph{uncertainty sets}~\citep{DBLP:journals/ior/NilimG05,DBLP:journals/mor/Iyengar05}.
We construct the robust model as an interval MDP that contains the true data-generating MDP with high probability.
Then, we provide, with high probability, a robust bound on the reach-avoid probabilities of the true MDP.
Using this result, we shield the SPI process to ensure it adheres to the safety requirements with high probability while still providing probabilistic guarantees on improvement within the shielded MDP due to SPI.
We illustrate our shielding methodology for SPI in~\Cref{fig:placeholder}.

We summarize our contributions as follows:
\begin{enumerate} 
    \item \textbf{Shields for safe offline RL:} We introduce a novel approach for safe offline RL by computing probabilistic shields on an interval MDP constructed from a given dataset. 
    \item \textbf{Shielding from a dataset}~(\cref{sec:shielding}): 
    We construct an interval MDP from the dataset and define a shield that restricts actions by thresholding worst-case reach-avoid probabilities, providing probabilistic safety guarantees.
    \item \textbf{Integration with SPI(BB)}~(\cref{sec:shield:spi}): We incorporate the shield into SPI, focusing specifically on SPIBB, ensuring that the resulting policies adhere to the shield.
\end{enumerate}
Experiments on four benchmarks evaluate the effectiveness of shielding in offline RL, specifically in the context of SPI, in terms of both safety and performance.
We show that, in most cases, shielded SPI (1) prevents the computation of unsafe policies and (2) improves overall and worst-case performance over standard SPI.

%% file: sections/02-preliminaries.tex
\section{Background}
\newcommand{\probdist}{\mu}
First, we introduce basic concepts and notation.
A \emph{probability distribution} over a finite set $X$ is a function $\probdist \colon X \to [0,1]$ such that $\sum_{x\in X} \probdist(x) = 1$.
The \emph{support} of a distribution $\probdist$ is $\supp(\probdist) = \{x \in X \mid \probdist(x) > 0\}$.
Let $\Delta(X) 
$ 
be the set of all probability distributions over $X$, and let $\powerset(X)$ be the \emph{powerset} of $X$.
For any predicate $\phi$, \emph{Iverson brackets}, denoted $[\phi]$, return $1$ if $\phi$ is true and $0$ otherwise.
For $f \colon X \to \Delta(Y)$ and similar, we write $f(y \mid x)$ to denote $f(x)(y)$.

\subsection{Markov decision processes}
A \emph{Markov decision process}~(MDP) models sequential decision-making problems in stochastic environments.
\begin{definition}
An MDP is a tuple~$\mdp = ( S, s_0, A, \T, \R, \gamma)$, where $S$ is the set of \emph{states} with $s_0 \in S$ the \emph{initial state}\footnote{We present a single initial state, but it generalizes to an initial state distribution.}, $A$ is the set of \emph{actions} from which $A(s)\subseteq A$ are \emph{available} at $s {\in} S$ , $\T \colon S \times A \hookrightarrow (S \to [0,1])$ is the partial \emph{transition function}%
, $\R \colon S \times A\to [-\Rmax, \Rmax]$ is the \emph{reward function}, with $\Rmax \in\mathbb{R}$, and $\gamma \in [0, 1)$ is the \emph{discount factor}. 
\end{definition}

A stochastic and memoryless \emph{policy}~$\pi \colon S \to \distr{A}$ selects the actions to execute.
The set of all stochastic and memoryless policies is $\Pi$. 
The \emph{value} of a policy $\pi\in\Pi$ from $s \in S$ is the expected infinite-horizon sum of discounted rewards:
\begin{equation}
    V_{\mdp}^\pi(s) = \E_{M,\pi}\left[\sum_{t=0}^{\infty}\gamma^t\R(s_t,a_t)\mid s_0=s\right],
\end{equation} 
where the expectation $ \E_{M,\pi}\left[\cdot\right]$ is taken over the transition probabilities of the Markov chain induced by following the policy $\pi$ in the MDP $M$~\citep{DBLP:books/wi/Puterman94}.
Its value follows from the value function $V_{\mdp}^\pi \colon S \to \mathbb{R}$, which is a fixed point of the Bellman equation~\citep{Bellman1957AMD}:
\(
    V_{\mdp}^\pi(s) = \sum_{a\in A} \pi(a\mid s) 
    \left[  \R(s,a)+\gamma\sum_{s'\in S} \T(s'\mid s,a)  V_{\mdp}^\pi(s')\right].
\)
Then, the \emph{performance} $\performance{\pi,\mdp} = V_{\mdp}^\pi(s_0)$ of a policy $\pi$ on MDP~$\mdp$ is the value of the initial state $s_0$. 
For a set of policies $\Pi'\subseteq\Pi$, a \emph{$\Pi'$-optimal policy} $\pi^* \in \Pi'$ maximizes the performance: 
\(\pi^* \in \argmax_{\pi\in\Pi'} \rho(\pi,\mdp).\)
Analogously to the value $V_{\mdp}^\pi$, the action-value $Q_{\mdp}^\pi \colon S \times A \to \mathbb{R}$ represents the expected value of taking action $a\in A$ in state $s\in S$.
\paragraph{Policy iteration.}
In this paper, we compute policies using variants of \emph{policy iteration}~\citep{DBLP:books/wi/Puterman94}, an algorithm for MDPs consisting of two steps: \emph{policy evaluation} and \emph{policy improvement}. 
In the evaluation step, we compute the action-value function $Q^\pi$. 
In the improvement step, we use $Q^\pi$ to update the policy $\pi$ greedily.
These steps are repeated iteratively, updating both $Q'^\pi\gets Q^\pi$ and $\pi'\gets \pi$, until $\pi'=\pi$ and convergence to an optimal policy is achieved~\citep{DBLP:books/wi/Puterman94}.

\paragraph{Datasets.}
In RL, we may find an optimal policy by learning the transition and reward function of the MDP. 
In this paper, we assume, without loss of generality, that the reward function is known.
In offline RL, we are restricted to using previously gathered data. 
A \emph{trajectory} is a finite sequence of states $\tau_H = (s_0,a_0,s_{1}, a_{1}, \ldots, s_H) \in (S \times A)^* \times S$ where $s_0$ is the initial state and where $\T(s_{t+1}\mid s_{t}, a_t) > 0$ for $0 \leq t < H$, and $t, H \in \mathbb{N}$.
A \emph{dataset} $\mathcal{D}$ consists of a finite number of trajectories.
Throughout the paper, 
we denote by $N_\dataset(s,a,s')$ and $N_\dataset(s,a)$ the number of occurrences of the transition triple $(s,a,s')$ and the state-action pair~$(s,a)$, respectively, in the given dataset of trajectories $\dataset$.

\subsection{Safe policy improvement}
\label{sec:spi}
The safe policy improvement (SPI) problem~\citep{DBLP:conf/aaai/ThomasTG15} in offline RL is to find a policy that improves over a given baseline policy with high probability. 
Specifically, from a dataset $\dataset$ and a baseline policy $\baseline$, SPI algorithms aim for a new policy $\improved$ that improves upon the baseline policy $\baseline$ on the true MDP $M^*$ with high probability:
\begin{equation}
    \label{eq:spi}
    \Pr\left( 
\rho(\improved, \mdp^*) 
\geq 
\rho(\baseline, \mdp^*) 
- \zeta \right) \geq 1 - \confspi,
\end{equation}
where $\zeta \in \mathbb{R}$ is the admissible error of the improvement approximation and $\confspi \in [0,1]$ signifies the required confidence level.
Note that $\zeta$ is a slack parameter, i.e., improvement is not necessarily achieved.
To find a solution to \cref{eq:spi}, we define the \emph{maximum likelihood} MDP~(\mlmdp), a tuple $\hat{\mdp} =  (S, A, \hat{\T}, \R, \gamma)$, with a transition function $\hat{\T}$ estimated via maximum likelihood as:
\begin{equation}
   \hat{\T}(s' \mid s,a) = [N_{\dataset}(s,a) > 0]\nicefrac{N_\dataset(s,a,s')}{N_{\dataset}(s,a)}.
\end{equation}
Note that in the \mlmdp, the uncertainty of the point estimate is not taken into account.
Given an \mlmdp and an \emph{error function} $\error \colon S \times A \to \mathbb{R}$, we can define an \emph{admissible set} of MDPs~\citep{DBLP:conf/nips/GhavamzadehPC16} that have a transition function with a bounded $\ell_1$ distance $||\cdot||_1$ from $\hat{T}$:
\begin{equation*}
    \Xi^{\error}_{\hat\mdp} = \left\{ \mdp 
    \mid \forall s,a \colon ||T^*(s,a)-\hat{T}(s,a)||_1 \leq \error(s,a) \right\}.
\end{equation*}
The $\error$ is typically defined from concentration inequalities given the dataset $\dataset$ such that the true MDP $M^*$ is captured by $\admissibleset$ with probability at least $1-\confspi$.
Then, finding the improved policy $\improved$ for \cref{eq:spi} is done through the following formulation~\citep{DBLP:conf/icml/LarocheTC19}, where we compute the $\Pi$-optimal policy $\improved$ on the \mlmdp~$\hat\mdp$ such that $\improved$ approximately improves over the baseline policy~$\baseline$ for all MDPs within the admissible set~$\admissibleset$:
\begin{equation*}
   \improved\in \argmax_{\pi \in \Pi} 
    \rho(\pi,\hat\mdp)
    \quad\text{s.t.}\quad \forall\mdp\in\admissibleset\colon 
    \rho(\pi,\mdp)
    \geq 
    \rho(\baseline,\mdp)
    -\zeta.
\end{equation*}
Since the admissible set~$\admissibleset$ captures the true MDP~$\mdp^*$ with probability $1-\confspi$, it provides a solution to~\cref{eq:spi}.
However, the amount of data required for each state-action pair to provide this guarantee is substantial~\citep{DBLP:conf/icml/LarocheTC19,DBLP:conf/ijcai/WienhoftSSDB023,DBLP:journals/corr/abs-2507-15532}.

\paragraph{SPIBB}
Providing a practical alternative,
\emph{safe policy improvement with baseline bootstrapping} (SPIBB)~\citep{DBLP:conf/icml/LarocheTC19} is an algorithm that restricts the policy improvement step to the baseline policy at state-action pairs with high uncertainty~\citep{DBLP:conf/nips/GhavamzadehPC16}. 
A threshold $\nwedge \in \mathbb{N}$ is specified to define the number of times a state-action pair must appear in the dataset $\dataset$ to allow improvement over the baseline.
To this end, we have a \emph{bootstrapped set} with \emph{uncertain} state-action pairs that occur less than $\nwedge$ times in the dataset~$\dataset$:
\begin{equation}
    \boots = \{(s,a)\in S\times A \mid N_{\mathcal{D}}(s,a)\leq \nwedge\}.
\end{equation}
Typically, $\nwedge$ is set as a hyperparameter, together with $\confspi$, and we quantify the improvement approximation $\zeta$~\citep{DBLP:conf/nips/GhavamzadehPC16,DBLP:conf/ijcai/WienhoftSSDB023}.
SPIBB optimizes the policy greedily over the expected reward~$Q^\pi_{\hat{M}}$ in the \mlmdp~$\hat M$,
subject to the constraint that $\forall(s,a) \in \boots\colon \pi'(a \mid s) = \baseline(a \mid s)$.
The probability mass of non-restricted actions is greedily assigned to the action that maximizes the action-value $Q^\pi_{\hat{M}}$. 
Given a state $s \in S$, let $A_{\not\boots}(s) = \{a \in A \mid (s,a) \notin \boots\} \subseteq A$ be the set of actions not in the bootstrapped set $\boots$.
Then, the update is:\footnote{In cases with no unique maximizing action $a'$, tie-breaking in the $\argmax$ is arbitrary.}
\begin{equation}
    \label{eq:spibb}
    \pi'(a \mid s) =
    \begin{cases}
       \baseline(a\mid s) & \text{ if } (s,a) \in \boots, \\
       \sum\limits_{a' \in A_{\not\boots}(s)} \baseline(a' \mid s) & \text{ if } a = \argmax\limits_{a' \in A_{\not\boots}(s)} Q^\pi_{\hat{M}}(s,a'),\\
       0 & \text{ otherwise.}
    \end{cases}
\end{equation}
Then, iterating $\pi \gets \pi'$ from the above~\cref{eq:spibb} until $\pi'=\pi$, starting from $\pi \gets \baseline$, then results in the policy $\improved$
that satisfies~\cref{eq:spi} with a $\zeta$ that depends on the chosen $N_\wedge$ and $\confspi$.
In the following, we assume that the baseline policy is known, but this assumption can be lifted by estimating it from the dataset~\citep{DBLP:conf/atal/SimaoLC20}, see~\cref{sec:est_baseline}.

%% file: sections/03-shielding.tex
\section{Probabilistic shielding in offline RL}
\label{sec:shielding}
We introduce \emph{probabilistic shielding} to offline RL, relying solely on a given dataset and knowledge of safe and unsafe states as input.
\paragraph{Type of shield.}
We focus on \emph{probabilistic shields}~\citep{court_probabilistic_2025,DBLP:conf/concur/0001KJSB20}, and only allow actions for which it is possible to satisfy the safety requirements with a certain probability from the current state onward.
Here, a shield is generally a function $A_{\shield} \colon S \to \powerset(A)$ mapping states to sets of allowed actions.
That is, we focus on \emph{pre-shielding}, where the shield restricts the set of actions the agent can choose from~\citep{cacm:shields}, rather than penalizing actions that lead to unsafe outcomes.
In offline RL, we do not have access to the true MDP~$\mdp^*$ and thus do not execute the policy before final deployment.
Instead, we restrict the actions that can be chosen during policy iteration, ensuring that the shield only allows actions for which it is possible to adhere to a probabilistic safety constraint.
In contrast to \emph{permissive policies}~\citep{draeger-et-al-tacas-2014,junges-et-al-tacas-2016}, we shield purely based on the probabilities of actions violating the safety requirements from the current state, and do not take into account past risks.
We refer to \citet{cacm:shields} for a detailed discussion on the various forms of shielding and their implications.
\paragraph{Overview.}
The process of constructing the shield from a dataset consists of three steps, described in the following subsections. 
\begin{enumerate}
    \item Construct an interval MDP model from the dataset $\dataset$. 
    \item Specify safety in terms of a reach-avoid specification on a subset of unsafe and target states. 
    \item Compute the worst-case probability of each state-action pair violating the safety specification to construct the shield. 
\end{enumerate}
\paragraph{Assumptions.}
To define a safety specification, we assume we are given a subset of \emph{unsafe} states $\unsafestates \subset S$ to avoid and a subset of \emph{target states} $\safestates \subset S$ to reach, with $\safestates \cap \unsafestates=\emptyset$.
    Additionally, we assume knowledge of the underlying graph of the true MDP~$M^*$, i.e., which transitions exist, but not the transition probabilities.
While this is not always assumed in SPI methods, it is a common assumption 
for estimating the IMDP and 
for analysis of the worst-case~\citep{DBLP:journals/mor/WiesemannKR13,DBLP:conf/nips/SuilenS0022}.
\subsection{Constructing an interval MDP}
Here, we construct an \emph{interval} MDP~(IMDP) with an interval-based \emph{uncertainty set} that captures the true MDP with high probability, similar to the admissible set for SPI.
\begin{definition}[IMDP]
An IMDP is a tuple $\mdp_{\intervals} = (S, s_0, A, \T_{\intervals}, R, \gamma)$. 
Here $S$, $s_0$, $A$, $\R$, and $\gamma$ are the same as in an MDP.
The partial \emph{interval transition function} is defined as $\T_{\intervals} \colon S \times A \hookrightarrow (S \to \intervals \cup \{0\})$, where $\intervals$ is the set of \emph{probability intervals} $\intervals = \{[l,u] \mid 0<l<u\leq1\}$.
\end{definition}
The interval transition function $\T_{\intervals}$ assigns either a probability interval or an exact probability of $0$ to each transition if the transition does not exist.
In contrast to the ball-shaped admissible set formalized using the $\ell_1$ norm, we construct a box-shaped uncertainty set, which allows an IMDP.
Note that in SPIBB, the admissible set is not constructed but solely defined to formalize a statistical guarantee.
Conversely, we explicitly construct the IMDP to compute a shield that provides robust safety guarantees based on the worst-case probabilities.
We construct IMDPs because they enable efficient analysis of worst-case probabilities, providing a lower bound on optimal probabilities, which we use to design the shield.

We construct interval ranges such that the true MDP $\mdp^*$ is captured with probability at least $1-\confimdp$.
We use the mode of a Dirichlet distribution as a point estimate~$\hat\T$~\citep{DBLP:conf/nips/SuilenS0022}.
We provide details in~\cref{app:map}.
We distribute $\confimdp$ over transitions via the union bound:
$\delta_{\T}= \nicefrac{\confimdp}{\sum_{s,a} |\supp(T(s,a))|}$
, where $|\supp(T(s,a))|$ is the number of successor states.
Given a point estimate $\hat\T$, and a confidence $\delta_{\T}$, we find the interval size $\errorrate_{\T}$ such that the true transition probability of $T^*$ falls within the interval with a probability at least $1-\delta_{\T}$. 
We find it each for transition $(s,a,s'_i)$, using Hoeffding’s inequality~\citep{Hoeffding}: 
\(
    \errorrate_{\T} = \sqrt{\nicefrac{\log\left[\nicefrac{2}{\delta_{\T}}\right]}{2N_{\dataset}(s,a)}}.
\)
Intuitively, since we know the sample sizes and have a desired confidence level $\delta_{\T}$, we determine the margin $\errorrate_{\T}$ around the point estimate.
Using this $\errorrate_{\T}$, we construct the interval for the state-action-state triplet $(s,a,s'_i)$ to be
$\hat{\T}_{\intervals}(s'_i\mid s,a) = [l, u]$, where $l = \max(\probmargin,\hat{\T}(s'_i \mid s, a)-\errorrate_{\T})$ and $u = \min(\hat{\T}(s'_i\mid s, a)+\errorrate_{\T},1)$.
Here, $\probmargin \in (0,1)$ is to ensure that the interval bounds $u$ and $l$ are in $[\probmargin, 1]$.
Repeating this process for every transition in the graph yields the IMDP $\hat{M}_{\intervals}$ conditioned on the dataset $\dataset$.
We achieve a guarantee on the level of the whole model by applying the \emph{union bound} across the transitions of the IMDP.%
\begin{corollary}
    \label{prop:pac:imdp}
    The transition probabilities $\T^*$ of the true MDP $M^*$ are within the transition intervals $\hat{\T_{\intervals}}$ of the interval MDP $\hat{M}_{\intervals}$ with probability at least $1 - \confimdp$. 
    That is, we have
    \(
    \Pr(\mdp^* \in \hat{\mdp}_{\intervals}) \geq 1-\confimdp.
    \)
\end{corollary}
Note that tighter intervals may be possible using other statistical methods~\citep{MWW25,DBLP:conf/ijcai/WienhoftSSDB023}.
Still, we use Hoeffding's inequality here for its simplicity.
To ensure robustness, in the subsequent steps, we take a pessimistic view of the interval MDP, accounting for the worst-case realization of the transition probabilities within the intervals to compute the shield, providing a lower bound with high probability.

\subsection{Defining the probabilistic specification}
The second step in the process of constructing a shield is defining a probabilistic safety specification. 
We encode our
specification with a \emph{temporal logic formula} $\varphi = \neg \unsafestates \;\mathsf{U}\; \safestates$ (read: ``not unsafe \emph{until} target'')~\citep{model_checking_bible}.
Specifically, we formulate a \emph{reach-avoid} formula, but an ``avoid'' formula may also be used.
Then, we formulate a probabilistic reach-avoid specification on the formula $\varphi$.

Given the estimated IMDP $\hat{\imdp}$ and the formula~$\varphi$, we define our specification based on the maximal reach-avoid probabilities.
Since the IMDP effectively defines ranges of probabilities, we robustly address this uncertainty by considering the worst case~\citep{DBLP:conf/cav/PuggelliLSS13}:
\[
\maxprob(s) = \max_{\pi\in \Pi} \min_{M \in \hat{M}_{\intervals}} \mathbb{P}_M^{\pi,\varphi}(s),
\]
where $\mathbb{P}_M^{\pi,\varphi}(s)$ is the probability of satisfying the formula $\varphi$ under policy~$\pi$ in the MDP~$M$, starting from the state $s\in S$%
~\citep{model_checking_bible}.
Intuitively, $\maxprob(s_0)$ represents, from the initial state~$s_0$, the probability that an agent will not enter any unsafe state before reaching a target state, when the agent selects actions that are optimal in terms of reach-avoid in the worst-case.
In the probabilistic \emph{specification}, the probability must satisfy a \emph{safety threshold}~$\threshold \in [0,1]$ from the initial state~$s_0$.
We disregard the past risk for tractability purposes~\citep{cacm:shields}, departing from semantics where the shield enforces the specification in the initial state of the MDP~\citep{court_probabilistic_2025,junges-et-al-tacas-2016}.
Instead, the shield ensures, with high probability, that for each state $s\in S$, the optimal worst-case reach-avoid probability adheres to the threshold. 
Specifically, the shield restricts actions such that, from that point onward and when ignoring past risks, the specification can still be satisfied.

\subsection{Computing and constructing the shield}
In the final step, we compute, for all states, the worst-case probabilities of violating the specification on the IMDP.
The worst-case probabilities can be induced by either the lower or upper bounds of the individual transition probabilities, hence the need for intervals.
We design the shield by thresholding the worst-case probabilities according to the specification.

\paragraph{Computing the probabilities.}
Given the specification and the IMDP $\hat{\imdp}$, we find the associated
function that map states $\rrV \colon S \to [0, 1]$ and state-action pairs $\rrQ \colon S \times A \to [0, 1]$ to maximal worst-case probabilities.
We set $\rrV(s) = 1$ if $s \in \safestates$ and $\rrV(s) = 0$ if $s \in \unsafestates$.
Then, we have 
for $s \in S \setminus (\safestates \cup \unsafestates)$,
the equations:
\begin{align}
\rrV(s) &= \max_{a\in A}\rrQ(s,a),\notag\\
\label{eq:probs:inner}
\rrQ(s,a) &= 
\inf\limits_{\underline{T}(s,a) \in T_\intervals(s,a) \cap \Delta(S)} \sum\limits_{s'\in S} \underline{T}(s' \mid s, a) \rrV(s')
.
\end{align}
The fixed points $\rrV^*$ and $\rrQ^*$, represent the \emph{robust optimal}, i.e., optimal worst-case reach-avoid probabilities, such that $\rrV^*(s) = \maxprob(s)$.
Specifically, the action-probability $\rrQ^*(s,a)$ represents, in expectation over successor states, the worst-case reach-avoid probability at any state $s\in S$ when taking action $a\in A$ and following the robust optimal reach-avoid actions afterwards~\citep{DBLP:conf/cav/PuggelliLSS13}.

For a given MDP $M$, the optimal (non-robust) reach-avoid probabilities $\mathbb{V}_M \colon S \to [0, 1]$ with 
$
\mathbb{V}_M^*(s) = \max_{\pi\in\Pi} \mathbb{P}_M^{\pi,\varphi}(s)$
and $\mathbb{Q}_M \colon S \times A \to [0, 1]$ with $\mathbb{Q}^*_M(s,a) = \sum_{s'} T(s'\mid s,a) \mathbb{V}^*(s')$ for $s \not\in \unsafestates \cup \safestates$
follow the computation of $\rrV,\rrQ$ by removing the inner infimum in \cref{eq:probs:inner} as $T$ is known.
We provide the following lower bound:
\begin{lemma}
    \label{lem:robustlb}
    Robust optimal reach-avoid probabilities $\rrV^*$ provide a lower bound on the true reach-avoid probabilities $\rV^*$ with probability at least $1-\confimdp$:
    \(
        \forall s \in S\colon \Pr\left(\rrV^*(s) \leq \rV^*(s)\right) \geq 1 - \confimdp.
    \)
\end{lemma}
\begin{corollary}
    \label{lem:robustQ}
    For $s \in S$ and $a \in A(s)$:
    \(\mathbb{Q}_{\hat{\imdp}}^*(s,a) \leq \mathbb{Q}_{M^*}(s,a)\).
\end{corollary}
\Cref{lem:robustQ} is due to \Cref{lem:robustlb} and \Cref{prop:pac:imdp}.
For IMDPs, the inner infimum of \Cref{eq:probs:inner} can be solved efficiently through a bisection method in combination with robust dynamic programming~\citep{DBLP:journals/ior/NilimG05,DBLP:journals/mor/Iyengar05,DBLP:conf/cdc/WolffTM12,DBLP:conf/cav/PuggelliLSS13}.
Notably, computing these worst-case maximal probabilities for IMDPs is supported out of the box by probabilistic model checkers, e.g., \prism~\citep{KNP11} and \storm~\citep{DBLP:conf/cav/DehnertJK017}. 
\paragraph{Designing the shield}
The \emph{safety threshold} $\threshold \in [0,1]$ is a settable parameter to bound the robust reach-avoid probabilities. 
Any state-action pair for which the probability is above this safety threshold is considered unsafe. 
Thus, we shield with respect to the robust optimal reach-avoid probabilities in the IMDP.
That is, with high probability, the shield allows any action for which it is possible to satisfy the threshold when ignoring past risk~\citep{cacm:shields}.
We define the notion of a safe action below.
\begin{definition}[$\threshold$-Safe action]
    \label{def:safe:action}
    An action $a \in A$ for a state $s \in S$ is a \emph{$\threshold$-safe action} in state $s$ if from state $s$, the probability of violating the reach-avoid specification in the IMDP is at most $\threshold$, i.e., it satisfies $\rrQ^*(s,a) > 1 - \threshold$.
    Otherwise, action $a$ is \emph{$\threshold$-unsafe} in state~$s$.
\end{definition}
\begin{assumption}[Feasibility]
We assume the threshold $\threshold \in [0,1]$ is \emph{feasible}, i.e., a $\threshold$-safe action is always available.
Otherwise, it is infeasible to satisfy the safety specification, and no shield exists.
That is, 
\(
\forall s \in S,\exists a \in A \colon 
\rrQ^*(s,a) > 1 - \threshold
.
\)
\end{assumption}
Next, we define the shield resulting from \cref{def:safe:action}.

\begin{definition}[$\threshold$-Shield]
    \label{def:shield}
    A \emph{$\threshold$-shield} is a function $A_{\shield} \colon S \to \powerset(A)$ that returns the set of $\threshold$-safe actions given a state $s \in S$, defined as: \[A_{\shield}(s) = \left\{a \in A \mid 
    \rrQ^*(s,a) > 1-\threshold
    \right\}\subseteq A.\]
    Conversely, let $\underline{A}_{\shield}(s) = A\setminus A_{\shield}(s)$ be the set of $\threshold$-unsafe actions such that $\forall s \colon \underline{A}_{\shield}(s) \cup {A}_{\shield}(s) = A$ and $\underline{A}_{\shield}(s) \cap {A}_{\shield}(s) = \emptyset$.
\end{definition}
Consequently, we define a policy to be safe if, for all states, its support is a subset of the actions allowed by the shield.
\begin{definition}[$\threshold$-Safe policy]
    \label{def:safepol}
    A policy $\pi \in \Pi$ is \emph{$\threshold$-safe} if its support contains only $\threshold$-safe actions: $\forall s \in S\colon\supp(\pi(s))  \subseteq A_{\shield}(s)$.
\end{definition}

\begin{remark}[Infeasibility]
    In practice, when there exist states where the threshold is infeasible, we select the action closest to the threshold, and we allow all actions with probabilities within a range $\buffer \in [0,1]$ around the probability of the closest action.
\end{remark}

\subsection{Safety guarantees}
Finally, we study the theoretical properties of the shield.
In the above, we constructed the IMDP to compute the shield.
Now, we argue about guarantees of the shield on the true MDP $M^*$. 
\begin{definition}
    The \emph{shielded-MDP} $\shielded{M} = \langle S, A, \shielded{\T}, R, \gamma \rangle$ for an MDP $M$ and shield $\shielded{A}$ is an MDP where $\shielded{\T}(s,a) = \T(s,a)$ if $a \in \shielded{A}(s)$ and $\shielded{\T}(s,a) = \bot$ otherwise, i.e., $a$ does not exist in the latter case.
\end{definition}
Next, we present a probabilistic guarantee that the shield contains only safe actions with high probability.
\begin{theorem}
    \label{thm:highprobsafe}
    Given a confidence level $\confshield$, and feasible threshold~$\threshold$, and conditioned on a dataset $\dataset$, a $\threshold$-shield~$A_{\shield}$ solely maps to $\threshold$-safe actions on the true MDP $\mdp^*$ with probability at least $1-\confshield$.
    Formally, starting from any state $s \in S$, the probability that the shielded action $a \in A_{\shield}(s)$ is also $\theta$-safe on the true MDP $\mdp^*$ is:
    \[
    \Pr\bigg(\mathbb{Q}^*_{\mdp^*}(s,a) \geq 1- \threshold \bigg) \geq 1 - \confshield.
    \]
\end{theorem}
\begin{proof}
    We defined an action $a \in A_{\shield}(s)$ to be a $\threshold$-safe action from state $s \in S$ if $\rrQ^*(s,a) \geq 1-\threshold$.
    From \cref{lem:robustQ}, we derive that with probability at least $1 - \confshield$, we have that $\rrQ^*(s,a) \leq \rQ^*(s,a)$. \qed
\end{proof}

The proof follows from \cref{lem:robustQ} and our definition of $\threshold$-safety with respect to robust reach-avoid probabilities.
Namely, condition on any $\dataset$, the IMDP $\hat\imdp$ contains the true MDP $\mdp^*$ with probability at least $1-\confshield$.
Thus, the shield on the IMDP also provides, with high probability, worst-case safety guarantees on the true MDP $\mdp^*$.
\begin{corollary}
    Safety with respect to the reach-avoid probabilities on the true MDP~$\mdp^*$ is lower bounded by the shielded MDP $\shielded{\mdp^*}$:
    $\forall s \in S \colon \mathbb{V}^*_{\shielded{\mdp^*}}(s) \leq \rV^*(s) $ with probability at least $1-\confimdp$.
\end{corollary}
Equality is not guaranteed in the above, so the shield may be more conservative than a shield using the true probabilities.
Nonetheless, \Cref{thm:highprobsafe} can be lifted to affect shielded policies; i.e., if any action deemed unsafe from the shield is ruled out, then the same probabilistic guarantee holds for the resulting policy. 

\section{Probabilistic Shielding for SPI}
\label{sec:shield:spi}
In this section, we integrate shields into SPIBB, yielding a new algorithm variant that we call \emph{shielded SPIBB} (\shieldspibb).
First, we describe how we ensure the safety of the policies resulting from \shieldspibb by restricting the policy space and shielding the baseline.
Then, we analyze the policies that \shieldspibb produces.

\paragraph{Other SPI algorithms.} While we focus the presentation on SPIBB, the shield may also be similarly applied to other SPI algorithms.
For instance, DUIPI~\citep{schneegassDuipi} is an SPI algorithm that penalizes the action-values of the \mlmdp with high uncertainty by computing their variances.
We also construct a shielded version of DUIPI, which we name \emph{shielded} DUIPI (DUIPI$_{\shield}$), which we describe in~\cref{app:duipi}.

\subsection{Applying the shield to SPIBB}
We enforce the shield on the baseline policy and incorporate it into the SPIBB algorithm.
Applying the shield ensures that resulting policies have no support over unsafe actions.

\paragraph{Shielding the baseline.}
Firstly, we enforce the shield on the baseline policy $\baseline$, creating a \emph{shielded baseline}~$\shieldedbaseline\in\Pi$.
\begin{definition}[Shielded baseline]
Given the baseline policy $\baseline \in \Pi$ and the shield $\shielded{A} \colon S \to \powerset(A)$, 
the \emph{shielded baseline} $\shieldedbaseline \in \Pi$ with probabilities for $a\in A$ given $s\in S$
, with $\underline{A}_{\shield}^s = \underline{A}_{\shield}(s)$ is defined as:
\[
\shieldedbaseline(a \mid s) = 
   \left[
   a \in A_{\shield}(s)
   \right]
   \left(\baseline(a\mid s) + \sum_{a' {\in} \underline{A}_{\shield}^s} \baseline(a'\mid s) \cdot \left(|A_{\shield}(s)|\right)^{-1}\right).
\]
\end{definition}
For each state $s\in S$, any $\threshold$-unsafe action $a \not\in\shielded{A}(s)$ has a probability of zero, and the probability mass of unsafe actions is distributed equally across the remaining safe actions.
That is, for safe actions $a \in \shielded{A}(s)$, we add the sum of probability mass $\baseline(a'\mid s)$ of all unsafe actions $a' \in \underline{A}_{\shield}(s)$ in the baseline policy, divided by the number of safe actions $|A_{\shield}(s)|$.

\begin{corollary}
    \label{lem:shieldedbaseline}
    The shielded baseline $\shieldedbaseline$ is $\threshold$-safe.
\end{corollary}
\paragraph{Shielding SPIBB} 
First, we amend the bootstrapped set $\boots$.
In addition to the state-action pairs that appear fewer than $N_{\wedge}$ times in the dataset~$\dataset$, we include state-action pairs deemed unsafe by the shield, regardless of their frequency of occurrence in the dataset.
\begin{equation*}
    \boots_{\shield} = \left\{(s,a)\in S\times A:N_{\mathcal{D}}(s,a)\leq N_{\wedge} \lor a \not\in A_{\shield}(s)\right\}
\end{equation*}
Consequently, for these state-action pairs, the shielded baseline is used, which, by~\Cref{lem:shieldedbaseline}, adheres to the shield.

Recall \cref{eq:spibb}.
Similarly, in \shieldspibb we optimize the policy subject to: $\forall(s,a) \in \boots_{\shield}\colon \pi'(a \mid s) = \shieldedbaseline(a \mid s)$.
Then, \shieldspibb addresses a SPI problem akin to~\Cref{eq:spi} where the objective is to find a policy $\shieldedimproved$ that approximately improves over the shielded baseline $\shieldedbaseline$ on the shielded \mlmdp~$\shielded{\hat{\mdp}}$ with high probability.
Thus, the policy improvement step of \shieldspibb that adheres to the shield is:
\begin{equation}
    \label{eq:shield:spibb}
    \pi'(a \mid s) =
    \begin{cases}
       \shieldedbaseline(a\mid s) & \text{ if } (s,a) \in \boots_{\shield}, \\
       \sum\limits_{a' \in A_{\not\boots_{\shield}}(s)} \shieldedbaseline(a' \mid s) & \text{ if } a = \argmax\limits_{a' \in A_{\not\shielded{\boots}}(s)} Q^\pi_{\shielded{\hat{M}}}(s,a'),\\
       0 & \text{ otherwise.}
    \end{cases}
\end{equation}
Intuitively, \shieldspibb greedily maximizes the action-value~$Q^\pi_{\shielded{\hat{M}}}$ across the safe actions of the shielded \mlmdp outside of the set $\shielded{\boots}$.
\subsection{Analysis of shielded SPIBB}
Let $\Pi_{\shielded{\boots}} = \left\{\pi \colon S \to \Delta(A) \mid \forall s,a\in\shielded{\boots} \colon \pi(a \mid s) = \shieldedbaseline(a\mid s)\right\}$ be the set of policies under the bootstrapped set $\shielded{\boots}$.
By definition and \cref{lem:shieldedbaseline}, any policy $\pi \in \Pi_{\shielded{\boots}}$ is $\threshold$-safe.
First, we formalize what policy iteration with~\cref{eq:shield:spibb} amounts to.
In \shieldspibb, policy iteration $\pi'\gets\pi$ with \cref{eq:shield:spibb}, starting from $\shieldedbaseline$ converges to $\shieldedimproved$ when $\pi'=\pi$, formalized below:
\begin{lemma}
    \shieldspibb finds a $\Pi_{\shielded{\boots}}$-optimal policy for the shielded \mlmdp~$\shielded{\hat\mdp}$, i.e., the policy:
    $\shieldedimproved \in \argmax_{\pi \in  \Pi_{\shielded{\boots}}} 
    \performance{\pi,\shielded{\hat\mdp}}$.
\end{lemma}

The above is a simple application of policy iteration to the set of shielded policies on the~\mlmdp~\citep{DBLP:books/wi/Puterman94,DBLP:conf/icml/LarocheTC19}.
Next, we formalize that \shieldspibb produces policies that adhere to the shield, i.e., are $\threshold$-safe.

\begin{lemma}[\shieldspibb is $\threshold$-safe]
    \label{lem:safespibb}
    \shieldspibb policies $\shieldedimproved \in \Pi$ are $\threshold$-safe in the true MDP $\mdp^*$ with probability
    at least $1 - \confshield $.
\end{lemma}

\Cref{lem:safespibb} is with high probability, since the shield is based on the worst-case that accounts for the true MDP~$\mdp^*$ with a probability of at least $1-\confshield$, following \cref{thm:highprobsafe}.
In summary, \shieldspibb provides a practical algorithm that produces policies that adhere to the shield constructed from the dataset in~\Cref{sec:shielding}.
\shieldspibb shields the baseline policy and, by adding unsafe state-action pairs to the bootstrapped set, greedily optimizes the policy over shielded actions with respect to the action-values of the shielded~\mlmdp~$\shielded{\hat\mdp}$.

%% file: sections/04-experiments.tex
\section{Experimental Evaluation}
In this section, we evaluate the shielded SPI methods against their non-shielded counterparts using four benchmarks.
We aim to answer the following three questions in the context of SPI methods:
\begin{questionenum}[leftmargin=*]
    \item How does shielding affect the sample complexity and performance of the resulting policies over the size of the dataset?\label{q1} 
    \item How does shielding affect worst-case performance, in terms of the average of the worst runs (measured in 1\%-CVaR)?\label{q2}
    \item How does shielding affect the safety of the resulting policies?\label{q3}
\end{questionenum}

\subsection{Experimental setup}
\paragraph{Methods.}
In addition to SPIBB and SPIBB$_{\shield}$ from~\cref{sec:shield:spi}, we consider DUIPI~\citep{schneegassDuipi}.
DUIPI is an SPI method that addresses uncertainty by computing the variance of the estimate of the transition function and propagating it through the action-values.
DUIPI penalizes the value of actions with uncertain outcomes by subtracting variance.
In a similar fashion to \shieldspibb, we create a shielded version DUIPI$_{\shield}$ by restricting the actions that can be chosen and by initializing with the shielded baseline~$\shieldedbaseline$.
We provide more details on DUIPI and DUIPI$_{\shield}$ in~\cref{app:duipi}.

\paragraph{Infrastructure.}
All experiments were conducted on the CPU of a Mac Studio with an M1 Ultra (128 GB unified memory).
We build on the implementation of SPIBB and DUIPI by~\citet{DBLP:conf/icaart/SchollDOU22a} and use their experimental framework in Python as the basis for ours.
The code is available at: \url{anonymous.4open.science/r/shielded-spi}.
We use the probabilistic model checker \storm~\citep{DBLP:conf/cav/DehnertJK017} to build and analyze the IMDP used to construct the shield.

\paragraph{Baseline policies.}
The baseline policies $\baseline \in \Pi$ are generated in a mixture with $\epsilon \in [0,1]$ such that 
\(\baseline = \epsilon\heuristic + (1-\epsilon)\tilde{\pi}\)
, where $\heuristic \in \Pi$ is a \emph{heuristic policy} and $\tilde{\pi}\in\Pi$ is the \emph{uniform random policy} with $\forall(s,a) \in S\times A\colon \tilde{\pi}(a \mid s) = \nicefrac{1}{|A|}$.

\paragraph{Experiments.}
Performing one experiment consists of the following steps, resembling steps in~\Cref{fig:placeholder}.
First, we gather a dataset~$\dataset$ by executing the baseline policy~$\baseline$ on the true MDP~$\mdp^*$.
Second, we estimate the baseline policy~$\baseline$ from the data~\citep{DBLP:conf/atal/SimaoLC20} as described in~\cref{sec:est_baseline}.
Third, we run the methods described above, using the data and the (shielded) baseline policy to compute a new policy $\pi$.
Finally, we evaluate the performance $\performance{\pi, \mdp^*}$ of the resulting policy $\pi$ on the true MDP~$\mdp^*$.
Each experiment is performed $1000$ times, and we report both the average performance and the 1\%-CVaR performance, which is the average of the worst $10$ runs. 
Important hyperparameters are listed in~\cref{tab:parameters}.

\begin{table}[tb]
    \caption{Dimensions and hyperparameters of the benchmarks.
    We set $\gamma=0.95$ and $\confimdp=0.1$ across.
    }
       \centering
    \begin{tabular}{lrrrrrrrrrr}
    \toprule
      Benchmarks  & $|S|$ & $|A|$ & $\nwedge$ & $\threshold$         & $\kappa$ & $\epsilon$        
      \\ \midrule
    \textit{Random MDPs}                 & 50 & 4 &3          & 0.2              & 0.05 &0.5             \\
    \textit{Wet Chicken}         & 26 & 5         & 7          & 0.2              & 0.05    &0.05          \\
    \textit{Frozen Lake}            & 64 & 4     & 3        & 0.2              & 0.02        &0.5     \\
    \textit{Pacman}              & 118k & 4        & 3         & 0.01             & 0.01       &0.5      \\
    \bottomrule
    \end{tabular}
        \label{tab:parameters}
\end{table}

\paragraph{Benchmarks}
The \emph{Random MDPs} benchmark~\citep{DBLP:conf/pkdd/NadjahiLC19} consists of randomly generated MDPs with $|S| = 100$ states, where each state has four actions that lead to four different states with a random probability. 
The \emph{Wet Chicken} benchmark~\citep{DBLP:conf/icann/HansU09} models an agent controlling a boat on a river adjacent to a waterfall, where the objective is to stay as close to the edge of the waterfall without falling off. 
The \emph{Frozen Lake} benchmark~\citep{towers2024gymnasiumstandardinterfacereinforcement} involves navigating an $8 \times 8$ grid to reach a goal state while avoiding hazardous hole states. 
The above benchmarks used in SPI are small; therefore, we construct
the larger \emph{Pacman} benchmark, inspired by~\citet{DBLP:conf/aaai/AlshiekhBEKNT18}; a $7 \times 7$ maze where the agent, able to move in the four cardinal directions, starts in a corner and aims to reach a goal state while avoiding two randomly moving ghosts. 
We provide dimensions in~\cref{tab:parameters} and detailed descriptions in~\cref{app:benchmarks} of the supplemental material.

\begin{figure*}[tb]
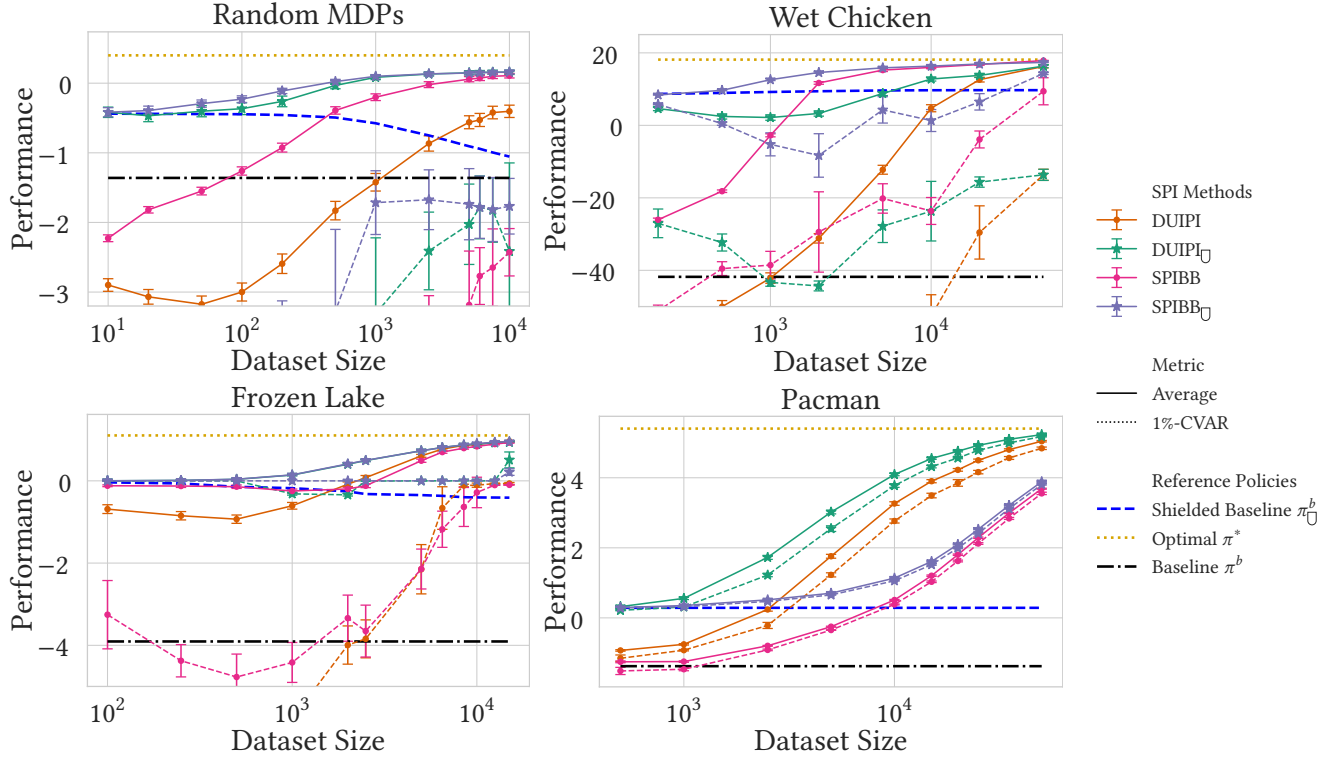

    \providecommand\mathdefault[1]{#1}
    \everymath=\expandafter{\the\everymath\displaystyle}
    \renewcommand\sffamily{}
    \begin{minipage}{0.8\textwidth}
    \begin{subfigure}{0.49\textwidth}
        \centering
        \resizebox{\linewidth}{!}{\input{images/results_grid_plot-0}}
    \end{subfigure}
    \begin{subfigure}{0.49\textwidth}
        \centering
        \resizebox{\linewidth}{!}{\input{images/results_grid_plot-1}}
    \end{subfigure}
    \\
    \begin{subfigure}{0.49\textwidth}
        \centering
        \resizebox{\linewidth}{!}{\input{images/results_grid_plot-3}}
    \end{subfigure}
    \begin{subfigure}{0.49\textwidth}
        \centering
        \resizebox{\linewidth}{!}{\input{images/results_grid_plot-2}}
    \end{subfigure}
    \end{minipage}
    \begin{minipage}{0.19\textwidth}
        \resizebox{\linewidth}{!}{\input{images/legend-results_grid_plot}}
    \end{minipage}
    \label{fig:results}
    \caption{Performance $\rho(\pi, \mdp^*)$ on the true MDP $\mdp^*$ of the policies $\pi \in \Pi$ produced by the shielded ($\star$) and non-shielded ($\bullet$) SPI methods plotted against the number of trajectories in the dataset~$\dataset$. We report average~(solid) and 1\%-CVaR~(dotted) across runs. The $x$-axis is on a log scale, and error bars denote $95$\% confidence intervals. The $y$-axis is clipped for visibility purposes; full ranges are in~\cref{app:plots}.
    The (shielded) baseline and optimal performances on the true MDP~$\mdp^*$ are included for reference. 
    }
\end{figure*}

\subsection{Results and analysis}

\Cref{fig:results} shows the average and 1\%-CVaR performance of the shielded and non-shielded DUIPI and SPIBB methods plotted against the size of the dataset $\dataset$, using a log-scale for the $x$-axis depicting dataset size. 
Note that for the episodic Pacman, Frozen Lake, and Random MDPs benchmarks, the dataset size refers to the number of trajectories in $\dataset$. 
In contrast, for the non-episodic Wet Chicken benchmark, it refers to the number of transitions in the dataset. 
Additionally, we plot the performance of the optimal, baseline, and shielded baseline policies.
Below, we analyze the results as answers to the three research questions.

\paragraph{\ref{q1}: Data complexity}
Both in terms of average and worst-case performance, \textbf{the shielded methods consistently outperform their non-shielded counterparts in low-data regimes}, particularly for smaller datasets. 
While both the shielded and non-shielded methods tend to converge similarly with sufficiently large datasets, the shielded versions achieve this level of performance with smaller datasets and generally converge faster.

The effectiveness of the shield on small datasets can be attributed to the fact that learning algorithms often struggle when data is scarce, as they lack sufficient information to accurately estimate the value of actions. 
This lack of data increases the likelihood that the agent will select unsafe or suboptimal actions. 
The shield mitigates this risk by ensuring that the optimization is applied only to safe actions, resulting in fewer unsafe outcomes, as discussed in the answer to Q3, which in turn leads to a higher overall reward. 
Moreover, SPI methods fall back to the baseline policy for state-action pairs with limited data. 
By shielding the baseline policy, the fallback will always select safe actions, providing an additional layer of protection and resulting in better overall performance with less data. 
In summary, \textbf{shielded SPI outperforms standard SPI with low-data regimes, generally converging faster.}

\paragraph{\ref{q2}: Worst-case performance.}
The 1\%-CVaR results show that the shielded methods generally outperform their non-shielded counterparts in terms of worst-case performance (see also \Cref{fig:results_cvar} in \Cref{app:plots} for the full range 1\%-CVaR results). 
Generally, we observe that the gains from the shield are larger in terms of 1\%-CVaR, indicating that the shield makes the algorithms more robust in the worst 1\% of runs, potentially when the datasets are low-quality.

The effect of the shield on the 1\%-CVaR is especially pronounced in the Pacman benchmark, which is also the largest. 
Here, the 1\%-CVaR performance of the shielded methods significantly outperforms their non-shielded counterparts.
Furthermore, we emphasize that the 1\%-CVaR performance of the shielded methods is superior to the average performance of the non-shielded methods across all dataset sizes.
Thus, the worst runs of the method with the shield are still better than the average of the runs without the shield.
For the other three benchmarks, we observe a general improvement in the shielded methods when comparing 1\%-CVaR. 
In summary, we conclude that \textbf{our shielding approach significantly improves worst-case outcomes as evidenced by 1\%-CVaR performance}.

\paragraph{\ref{q3}: Safety.}
In the four benchmarks used in our experiments, all unsafe states are associated with a negative reward, while all safe states have a reward of zero or greater. 
This means that if a policy's performance is negative, it has a positive probability of reaching an unsafe state.

In three of the four benchmarks, we observe that the average performance of shielded SPI methods remains above zero.
In the random MDPs benchmark, the shielded methods, as well as the shielded baseline, start with a performance slightly below zero.
Here, many actions inherently carry a non-negligible probability of eventually leading to a trap, making it challenging, if not impossible, to identify a policy that completely guarantees safety.
Still, we observe that the shielded SPI methods achieve a better performance.

The Frozen Lake benchmark demonstrates that the shield helps the SPI methods avoid unsafe actions.
Here, shielded methods almost never produce policies that yield negative rewards, indicating that they consistently avoid unsafe behaviors, regardless of dataset size.
With small datasets, the shielded methods converge to conservative but safe policies that yield a reward of just above zero.
While this policy does not reach the goal state, and thus does not earn a positive reward, it also avoids any unsafe states.
In contrast, the non-shielded methods learn policies that not only fail to reach the goal state but also contain unsafe actions.
Even with larger datasets, shielded methods retain this advantage, achieving higher rewards because they carry a lower risk of unsafe outcomes than their non-shielded counterparts. 
Achieving the same performance as the non-shielded methods with larger datasets demonstrates that, in our benchmarks, the shield does not hinder the computation of near-optimal policies.
In summary, \textbf{the shielded methods achieve safer policies in low-data regimes while remaining competitive with larger datasets.}

%% file: images/legend-results_grid_plot.tex
\begingroup%
\makeatletter%
\begin{pgfpicture}%
\pgfpathrectangle{\pgfpointorigin}{\pgfqpoint{2.705000in}{4.420833in}}%
\pgfusepath{use as bounding box, clip}%
\begin{pgfscope}%
\pgfsetbuttcap%
\pgfsetmiterjoin%
\definecolor{currentfill}{rgb}{1.000000,1.000000,1.000000}%
\pgfsetfillcolor{currentfill}%
\pgfsetlinewidth{0.000000pt}%
\definecolor{currentstroke}{rgb}{1.000000,1.000000,1.000000}%
\pgfsetstrokecolor{currentstroke}%
\pgfsetdash{}{0pt}%
\pgfpathmoveto{\pgfqpoint{-0.000000in}{0.000000in}}%
\pgfpathlineto{\pgfqpoint{2.705000in}{0.000000in}}%
\pgfpathlineto{\pgfqpoint{2.705000in}{4.420833in}}%
\pgfpathlineto{\pgfqpoint{-0.000000in}{4.420833in}}%
\pgfpathlineto{\pgfqpoint{-0.000000in}{0.000000in}}%
\pgfpathclose%
\pgfusepath{fill}%
\end{pgfscope}%
\begin{pgfscope}%
\definecolor{textcolor}{rgb}{0.150000,0.150000,0.150000}%
\pgfsetstrokecolor{textcolor}%
\pgfsetfillcolor{textcolor}%
\pgftext[x=0.718975in,y=4.140550in,left,base]{\color{textcolor}{\sffamily\fontsize{15.000000}{18.000000}\selectfont\catcode`\^=\active\def^{\ifmmode\sp\else\^{}\fi}\catcode`\%=\active\def
\end{pgfscope}%
\begin{pgfscope}%
\pgfsetbuttcap%
\pgfsetroundjoin%
\pgfsetlinewidth{1.204500pt}%
\definecolor{currentstroke}{rgb}{0.850980,0.372549,0.007843}%
\pgfsetstrokecolor{currentstroke}%
\pgfsetdash{}{0pt}%
\pgfpathmoveto{\pgfqpoint{0.343975in}{3.803514in}}%
\pgfpathlineto{\pgfqpoint{0.343975in}{4.011847in}}%
\pgfusepath{stroke}%
\end{pgfscope}%
\begin{pgfscope}%
\pgfsetbuttcap%
\pgfsetroundjoin%
\definecolor{currentfill}{rgb}{0.850980,0.372549,0.007843}%
\pgfsetfillcolor{currentfill}%
\pgfsetlinewidth{1.003750pt}%
\definecolor{currentstroke}{rgb}{0.850980,0.372549,0.007843}%
\pgfsetstrokecolor{currentstroke}%
\pgfsetdash{}{0pt}%
\pgfsys@defobject{currentmarker}{\pgfqpoint{-0.055556in}{-0.000000in}}{\pgfqpoint{0.055556in}{0.000000in}}{%
\pgfpathmoveto{\pgfqpoint{0.055556in}{-0.000000in}}%
\pgfpathlineto{\pgfqpoint{-0.055556in}{0.000000in}}%
\pgfusepath{stroke,fill}%
}%
\begin{pgfscope}%
\pgfsys@transformshift{0.343975in}{3.803514in}%
\pgfsys@useobject{currentmarker}{}%
\end{pgfscope}%
\end{pgfscope}%
\begin{pgfscope}%
\pgfsetbuttcap%
\pgfsetroundjoin%
\definecolor{currentfill}{rgb}{0.850980,0.372549,0.007843}%
\pgfsetfillcolor{currentfill}%
\pgfsetlinewidth{1.003750pt}%
\definecolor{currentstroke}{rgb}{0.850980,0.372549,0.007843}%
\pgfsetstrokecolor{currentstroke}%
\pgfsetdash{}{0pt}%
\pgfsys@defobject{currentmarker}{\pgfqpoint{-0.055556in}{-0.000000in}}{\pgfqpoint{0.055556in}{0.000000in}}{%
\pgfpathmoveto{\pgfqpoint{0.055556in}{-0.000000in}}%
\pgfpathlineto{\pgfqpoint{-0.055556in}{0.000000in}}%
\pgfusepath{stroke,fill}%
}%
\begin{pgfscope}%
\pgfsys@transformshift{0.343975in}{4.011847in}%
\pgfsys@useobject{currentmarker}{}%
\end{pgfscope}%
\end{pgfscope}%
\begin{pgfscope}%
\pgfsetroundcap%
\pgfsetroundjoin%
\pgfsetlinewidth{1.204500pt}%
\definecolor{currentstroke}{rgb}{0.850980,0.372549,0.007843}%
\pgfsetstrokecolor{currentstroke}%
\pgfsetdash{}{0pt}%
\pgfpathmoveto{\pgfqpoint{0.135642in}{3.907681in}}%
\pgfpathlineto{\pgfqpoint{0.552309in}{3.907681in}}%
\pgfusepath{stroke}%
\end{pgfscope}%
\begin{pgfscope}%
\pgfsetbuttcap%
\pgfsetroundjoin%
\definecolor{currentfill}{rgb}{0.850980,0.372549,0.007843}%
\pgfsetfillcolor{currentfill}%
\pgfsetlinewidth{1.003750pt}%
\definecolor{currentstroke}{rgb}{0.850980,0.372549,0.007843}%
\pgfsetstrokecolor{currentstroke}%
\pgfsetdash{}{0pt}%
\pgfsys@defobject{currentmarker}{\pgfqpoint{-0.027778in}{-0.027778in}}{\pgfqpoint{0.027778in}{0.027778in}}{%
\pgfpathmoveto{\pgfqpoint{0.000000in}{-0.027778in}}%
\pgfpathcurveto{\pgfqpoint{0.007367in}{-0.027778in}}{\pgfqpoint{0.014433in}{-0.024851in}}{\pgfqpoint{0.019642in}{-0.019642in}}%
\pgfpathcurveto{\pgfqpoint{0.024851in}{-0.014433in}}{\pgfqpoint{0.027778in}{-0.007367in}}{\pgfqpoint{0.027778in}{0.000000in}}%
\pgfpathcurveto{\pgfqpoint{0.027778in}{0.007367in}}{\pgfqpoint{0.024851in}{0.014433in}}{\pgfqpoint{0.019642in}{0.019642in}}%
\pgfpathcurveto{\pgfqpoint{0.014433in}{0.024851in}}{\pgfqpoint{0.007367in}{0.027778in}}{\pgfqpoint{0.000000in}{0.027778in}}%
\pgfpathcurveto{\pgfqpoint{-0.007367in}{0.027778in}}{\pgfqpoint{-0.014433in}{0.024851in}}{\pgfqpoint{-0.019642in}{0.019642in}}%
\pgfpathcurveto{\pgfqpoint{-0.024851in}{0.014433in}}{\pgfqpoint{-0.027778in}{0.007367in}}{\pgfqpoint{-0.027778in}{0.000000in}}%
\pgfpathcurveto{\pgfqpoint{-0.027778in}{-0.007367in}}{\pgfqpoint{-0.024851in}{-0.014433in}}{\pgfqpoint{-0.019642in}{-0.019642in}}%
\pgfpathcurveto{\pgfqpoint{-0.014433in}{-0.024851in}}{\pgfqpoint{-0.007367in}{-0.027778in}}{\pgfqpoint{0.000000in}{-0.027778in}}%
\pgfpathlineto{\pgfqpoint{0.000000in}{-0.027778in}}%
\pgfpathclose%
\pgfusepath{stroke,fill}%
}%
\begin{pgfscope}%
\pgfsys@transformshift{0.343975in}{3.907681in}%
\pgfsys@useobject{currentmarker}{}%
\end{pgfscope}%
\end{pgfscope}%
\begin{pgfscope}%
\definecolor{textcolor}{rgb}{0.150000,0.150000,0.150000}%
\pgfsetstrokecolor{textcolor}%
\pgfsetfillcolor{textcolor}%
\pgftext[x=0.718975in,y=3.834764in,left,base]{\color{textcolor}{\sffamily\fontsize{15.000000}{18.000000}\selectfont\catcode`\^=\active\def^{\ifmmode\sp\else\^{}\fi}\catcode`\%=\active\def
\end{pgfscope}%
\begin{pgfscope}%
\pgfsetbuttcap%
\pgfsetroundjoin%
\pgfsetlinewidth{1.204500pt}%
\definecolor{currentstroke}{rgb}{0.105882,0.619608,0.466667}%
\pgfsetstrokecolor{currentstroke}%
\pgfsetdash{}{0pt}%
\pgfpathmoveto{\pgfqpoint{0.343975in}{3.497728in}}%
\pgfpathlineto{\pgfqpoint{0.343975in}{3.706061in}}%
\pgfusepath{stroke}%
\end{pgfscope}%
\begin{pgfscope}%
\pgfsetbuttcap%
\pgfsetroundjoin%
\definecolor{currentfill}{rgb}{0.105882,0.619608,0.466667}%
\pgfsetfillcolor{currentfill}%
\pgfsetlinewidth{1.003750pt}%
\definecolor{currentstroke}{rgb}{0.105882,0.619608,0.466667}%
\pgfsetstrokecolor{currentstroke}%
\pgfsetdash{}{0pt}%
\pgfsys@defobject{currentmarker}{\pgfqpoint{-0.055556in}{-0.000000in}}{\pgfqpoint{0.055556in}{0.000000in}}{%
\pgfpathmoveto{\pgfqpoint{0.055556in}{-0.000000in}}%
\pgfpathlineto{\pgfqpoint{-0.055556in}{0.000000in}}%
\pgfusepath{stroke,fill}%
}%
\begin{pgfscope}%
\pgfsys@transformshift{0.343975in}{3.497728in}%
\pgfsys@useobject{currentmarker}{}%
\end{pgfscope}%
\end{pgfscope}%
\begin{pgfscope}%
\pgfsetbuttcap%
\pgfsetroundjoin%
\definecolor{currentfill}{rgb}{0.105882,0.619608,0.466667}%
\pgfsetfillcolor{currentfill}%
\pgfsetlinewidth{1.003750pt}%
\definecolor{currentstroke}{rgb}{0.105882,0.619608,0.466667}%
\pgfsetstrokecolor{currentstroke}%
\pgfsetdash{}{0pt}%
\pgfsys@defobject{currentmarker}{\pgfqpoint{-0.055556in}{-0.000000in}}{\pgfqpoint{0.055556in}{0.000000in}}{%
\pgfpathmoveto{\pgfqpoint{0.055556in}{-0.000000in}}%
\pgfpathlineto{\pgfqpoint{-0.055556in}{0.000000in}}%
\pgfusepath{stroke,fill}%
}%
\begin{pgfscope}%
\pgfsys@transformshift{0.343975in}{3.706061in}%
\pgfsys@useobject{currentmarker}{}%
\end{pgfscope}%
\end{pgfscope}%
\begin{pgfscope}%
\pgfsetroundcap%
\pgfsetroundjoin%
\pgfsetlinewidth{1.204500pt}%
\definecolor{currentstroke}{rgb}{0.105882,0.619608,0.466667}%
\pgfsetstrokecolor{currentstroke}%
\pgfsetdash{}{0pt}%
\pgfpathmoveto{\pgfqpoint{0.135642in}{3.601895in}}%
\pgfpathlineto{\pgfqpoint{0.552309in}{3.601895in}}%
\pgfusepath{stroke}%
\end{pgfscope}%
\begin{pgfscope}%
\pgfsetbuttcap%
\pgfsetbeveljoin%
\definecolor{currentfill}{rgb}{0.105882,0.619608,0.466667}%
\pgfsetfillcolor{currentfill}%
\pgfsetlinewidth{1.003750pt}%
\definecolor{currentstroke}{rgb}{0.105882,0.619608,0.466667}%
\pgfsetstrokecolor{currentstroke}%
\pgfsetdash{}{0pt}%
\pgfsys@defobject{currentmarker}{\pgfqpoint{-0.052836in}{-0.044945in}}{\pgfqpoint{0.052836in}{0.055556in}}{%
\pgfpathmoveto{\pgfqpoint{0.000000in}{0.055556in}}%
\pgfpathlineto{\pgfqpoint{-0.012473in}{0.017168in}}%
\pgfpathlineto{\pgfqpoint{-0.052836in}{0.017168in}}%
\pgfpathlineto{\pgfqpoint{-0.020182in}{-0.006557in}}%
\pgfpathlineto{\pgfqpoint{-0.032655in}{-0.044945in}}%
\pgfpathlineto{\pgfqpoint{-0.000000in}{-0.021220in}}%
\pgfpathlineto{\pgfqpoint{0.032655in}{-0.044945in}}%
\pgfpathlineto{\pgfqpoint{0.020182in}{-0.006557in}}%
\pgfpathlineto{\pgfqpoint{0.052836in}{0.017168in}}%
\pgfpathlineto{\pgfqpoint{0.012473in}{0.017168in}}%
\pgfpathlineto{\pgfqpoint{0.000000in}{0.055556in}}%
\pgfpathclose%
\pgfusepath{stroke,fill}%
}%
\begin{pgfscope}%
\pgfsys@transformshift{0.343975in}{3.601895in}%
\pgfsys@useobject{currentmarker}{}%
\end{pgfscope}%
\end{pgfscope}%
\begin{pgfscope}%
\definecolor{textcolor}{rgb}{0.150000,0.150000,0.150000}%
\pgfsetstrokecolor{textcolor}%
\pgfsetfillcolor{textcolor}%
\pgftext[x=0.718975in,y=3.528978in,left,base]{\color{textcolor}{\sffamily\fontsize{15.000000}{18.000000}\selectfont\catcode`\^=\active\def^{\ifmmode\sp\else\^{}\fi}\catcode`\%=\active\def
\end{pgfscope}%
\begin{pgfscope}%
\pgfsetbuttcap%
\pgfsetroundjoin%
\pgfsetlinewidth{1.204500pt}%
\definecolor{currentstroke}{rgb}{0.905882,0.160784,0.541176}%
\pgfsetstrokecolor{currentstroke}%
\pgfsetdash{}{0pt}%
\pgfpathmoveto{\pgfqpoint{0.343975in}{3.191942in}}%
\pgfpathlineto{\pgfqpoint{0.343975in}{3.400276in}}%
\pgfusepath{stroke}%
\end{pgfscope}%
\begin{pgfscope}%
\pgfsetbuttcap%
\pgfsetroundjoin%
\definecolor{currentfill}{rgb}{0.905882,0.160784,0.541176}%
\pgfsetfillcolor{currentfill}%
\pgfsetlinewidth{1.003750pt}%
\definecolor{currentstroke}{rgb}{0.905882,0.160784,0.541176}%
\pgfsetstrokecolor{currentstroke}%
\pgfsetdash{}{0pt}%
\pgfsys@defobject{currentmarker}{\pgfqpoint{-0.055556in}{-0.000000in}}{\pgfqpoint{0.055556in}{0.000000in}}{%
\pgfpathmoveto{\pgfqpoint{0.055556in}{-0.000000in}}%
\pgfpathlineto{\pgfqpoint{-0.055556in}{0.000000in}}%
\pgfusepath{stroke,fill}%
}%
\begin{pgfscope}%
\pgfsys@transformshift{0.343975in}{3.191942in}%
\pgfsys@useobject{currentmarker}{}%
\end{pgfscope}%
\end{pgfscope}%
\begin{pgfscope}%
\pgfsetbuttcap%
\pgfsetroundjoin%
\definecolor{currentfill}{rgb}{0.905882,0.160784,0.541176}%
\pgfsetfillcolor{currentfill}%
\pgfsetlinewidth{1.003750pt}%
\definecolor{currentstroke}{rgb}{0.905882,0.160784,0.541176}%
\pgfsetstrokecolor{currentstroke}%
\pgfsetdash{}{0pt}%
\pgfsys@defobject{currentmarker}{\pgfqpoint{-0.055556in}{-0.000000in}}{\pgfqpoint{0.055556in}{0.000000in}}{%
\pgfpathmoveto{\pgfqpoint{0.055556in}{-0.000000in}}%
\pgfpathlineto{\pgfqpoint{-0.055556in}{0.000000in}}%
\pgfusepath{stroke,fill}%
}%
\begin{pgfscope}%
\pgfsys@transformshift{0.343975in}{3.400276in}%
\pgfsys@useobject{currentmarker}{}%
\end{pgfscope}%
\end{pgfscope}%
\begin{pgfscope}%
\pgfsetroundcap%
\pgfsetroundjoin%
\pgfsetlinewidth{1.204500pt}%
\definecolor{currentstroke}{rgb}{0.905882,0.160784,0.541176}%
\pgfsetstrokecolor{currentstroke}%
\pgfsetdash{}{0pt}%
\pgfpathmoveto{\pgfqpoint{0.135642in}{3.296109in}}%
\pgfpathlineto{\pgfqpoint{0.552309in}{3.296109in}}%
\pgfusepath{stroke}%
\end{pgfscope}%
\begin{pgfscope}%
\pgfsetbuttcap%
\pgfsetroundjoin%
\definecolor{currentfill}{rgb}{0.905882,0.160784,0.541176}%
\pgfsetfillcolor{currentfill}%
\pgfsetlinewidth{1.003750pt}%
\definecolor{currentstroke}{rgb}{0.905882,0.160784,0.541176}%
\pgfsetstrokecolor{currentstroke}%
\pgfsetdash{}{0pt}%
\pgfsys@defobject{currentmarker}{\pgfqpoint{-0.027778in}{-0.027778in}}{\pgfqpoint{0.027778in}{0.027778in}}{%
\pgfpathmoveto{\pgfqpoint{0.000000in}{-0.027778in}}%
\pgfpathcurveto{\pgfqpoint{0.007367in}{-0.027778in}}{\pgfqpoint{0.014433in}{-0.024851in}}{\pgfqpoint{0.019642in}{-0.019642in}}%
\pgfpathcurveto{\pgfqpoint{0.024851in}{-0.014433in}}{\pgfqpoint{0.027778in}{-0.007367in}}{\pgfqpoint{0.027778in}{0.000000in}}%
\pgfpathcurveto{\pgfqpoint{0.027778in}{0.007367in}}{\pgfqpoint{0.024851in}{0.014433in}}{\pgfqpoint{0.019642in}{0.019642in}}%
\pgfpathcurveto{\pgfqpoint{0.014433in}{0.024851in}}{\pgfqpoint{0.007367in}{0.027778in}}{\pgfqpoint{0.000000in}{0.027778in}}%
\pgfpathcurveto{\pgfqpoint{-0.007367in}{0.027778in}}{\pgfqpoint{-0.014433in}{0.024851in}}{\pgfqpoint{-0.019642in}{0.019642in}}%
\pgfpathcurveto{\pgfqpoint{-0.024851in}{0.014433in}}{\pgfqpoint{-0.027778in}{0.007367in}}{\pgfqpoint{-0.027778in}{0.000000in}}%
\pgfpathcurveto{\pgfqpoint{-0.027778in}{-0.007367in}}{\pgfqpoint{-0.024851in}{-0.014433in}}{\pgfqpoint{-0.019642in}{-0.019642in}}%
\pgfpathcurveto{\pgfqpoint{-0.014433in}{-0.024851in}}{\pgfqpoint{-0.007367in}{-0.027778in}}{\pgfqpoint{0.000000in}{-0.027778in}}%
\pgfpathlineto{\pgfqpoint{0.000000in}{-0.027778in}}%
\pgfpathclose%
\pgfusepath{stroke,fill}%
}%
\begin{pgfscope}%
\pgfsys@transformshift{0.343975in}{3.296109in}%
\pgfsys@useobject{currentmarker}{}%
\end{pgfscope}%
\end{pgfscope}%
\begin{pgfscope}%
\definecolor{textcolor}{rgb}{0.150000,0.150000,0.150000}%
\pgfsetstrokecolor{textcolor}%
\pgfsetfillcolor{textcolor}%
\pgftext[x=0.718975in,y=3.223192in,left,base]{\color{textcolor}{\sffamily\fontsize{15.000000}{18.000000}\selectfont\catcode`\^=\active\def^{\ifmmode\sp\else\^{}\fi}\catcode`\%=\active\def
\end{pgfscope}%
\begin{pgfscope}%
\pgfsetbuttcap%
\pgfsetroundjoin%
\pgfsetlinewidth{1.204500pt}%
\definecolor{currentstroke}{rgb}{0.458824,0.439216,0.701961}%
\pgfsetstrokecolor{currentstroke}%
\pgfsetdash{}{0pt}%
\pgfpathmoveto{\pgfqpoint{0.343975in}{2.886156in}}%
\pgfpathlineto{\pgfqpoint{0.343975in}{3.094490in}}%
\pgfusepath{stroke}%
\end{pgfscope}%
\begin{pgfscope}%
\pgfsetbuttcap%
\pgfsetroundjoin%
\definecolor{currentfill}{rgb}{0.458824,0.439216,0.701961}%
\pgfsetfillcolor{currentfill}%
\pgfsetlinewidth{1.003750pt}%
\definecolor{currentstroke}{rgb}{0.458824,0.439216,0.701961}%
\pgfsetstrokecolor{currentstroke}%
\pgfsetdash{}{0pt}%
\pgfsys@defobject{currentmarker}{\pgfqpoint{-0.055556in}{-0.000000in}}{\pgfqpoint{0.055556in}{0.000000in}}{%
\pgfpathmoveto{\pgfqpoint{0.055556in}{-0.000000in}}%
\pgfpathlineto{\pgfqpoint{-0.055556in}{0.000000in}}%
\pgfusepath{stroke,fill}%
}%
\begin{pgfscope}%
\pgfsys@transformshift{0.343975in}{2.886156in}%
\pgfsys@useobject{currentmarker}{}%
\end{pgfscope}%
\end{pgfscope}%
\begin{pgfscope}%
\pgfsetbuttcap%
\pgfsetroundjoin%
\definecolor{currentfill}{rgb}{0.458824,0.439216,0.701961}%
\pgfsetfillcolor{currentfill}%
\pgfsetlinewidth{1.003750pt}%
\definecolor{currentstroke}{rgb}{0.458824,0.439216,0.701961}%
\pgfsetstrokecolor{currentstroke}%
\pgfsetdash{}{0pt}%
\pgfsys@defobject{currentmarker}{\pgfqpoint{-0.055556in}{-0.000000in}}{\pgfqpoint{0.055556in}{0.000000in}}{%
\pgfpathmoveto{\pgfqpoint{0.055556in}{-0.000000in}}%
\pgfpathlineto{\pgfqpoint{-0.055556in}{0.000000in}}%
\pgfusepath{stroke,fill}%
}%
\begin{pgfscope}%
\pgfsys@transformshift{0.343975in}{3.094490in}%
\pgfsys@useobject{currentmarker}{}%
\end{pgfscope}%
\end{pgfscope}%
\begin{pgfscope}%
\pgfsetroundcap%
\pgfsetroundjoin%
\pgfsetlinewidth{1.204500pt}%
\definecolor{currentstroke}{rgb}{0.458824,0.439216,0.701961}%
\pgfsetstrokecolor{currentstroke}%
\pgfsetdash{}{0pt}%
\pgfpathmoveto{\pgfqpoint{0.135642in}{2.990323in}}%
\pgfpathlineto{\pgfqpoint{0.552309in}{2.990323in}}%
\pgfusepath{stroke}%
\end{pgfscope}%
\begin{pgfscope}%
\pgfsetbuttcap%
\pgfsetbeveljoin%
\definecolor{currentfill}{rgb}{0.458824,0.439216,0.701961}%
\pgfsetfillcolor{currentfill}%
\pgfsetlinewidth{1.003750pt}%
\definecolor{currentstroke}{rgb}{0.458824,0.439216,0.701961}%
\pgfsetstrokecolor{currentstroke}%
\pgfsetdash{}{0pt}%
\pgfsys@defobject{currentmarker}{\pgfqpoint{-0.052836in}{-0.044945in}}{\pgfqpoint{0.052836in}{0.055556in}}{%
\pgfpathmoveto{\pgfqpoint{0.000000in}{0.055556in}}%
\pgfpathlineto{\pgfqpoint{-0.012473in}{0.017168in}}%
\pgfpathlineto{\pgfqpoint{-0.052836in}{0.017168in}}%
\pgfpathlineto{\pgfqpoint{-0.020182in}{-0.006557in}}%
\pgfpathlineto{\pgfqpoint{-0.032655in}{-0.044945in}}%
\pgfpathlineto{\pgfqpoint{-0.000000in}{-0.021220in}}%
\pgfpathlineto{\pgfqpoint{0.032655in}{-0.044945in}}%
\pgfpathlineto{\pgfqpoint{0.020182in}{-0.006557in}}%
\pgfpathlineto{\pgfqpoint{0.052836in}{0.017168in}}%
\pgfpathlineto{\pgfqpoint{0.012473in}{0.017168in}}%
\pgfpathlineto{\pgfqpoint{0.000000in}{0.055556in}}%
\pgfpathclose%
\pgfusepath{stroke,fill}%
}%
\begin{pgfscope}%
\pgfsys@transformshift{0.343975in}{2.990323in}%
\pgfsys@useobject{currentmarker}{}%
\end{pgfscope}%
\end{pgfscope}%
\begin{pgfscope}%
\definecolor{textcolor}{rgb}{0.150000,0.150000,0.150000}%
\pgfsetstrokecolor{textcolor}%
\pgfsetfillcolor{textcolor}%
\pgftext[x=0.718975in,y=2.917406in,left,base]{\color{textcolor}{\sffamily\fontsize{15.000000}{18.000000}\selectfont\catcode`\^=\active\def^{\ifmmode\sp\else\^{}\fi}\catcode`\%=\active\def
\end{pgfscope}%
\begin{pgfscope}%
\definecolor{textcolor}{rgb}{0.150000,0.150000,0.150000}%
\pgfsetstrokecolor{textcolor}%
\pgfsetfillcolor{textcolor}%
\pgftext[x=0.718975in,y=2.305835in,left,base]{\color{textcolor}{\sffamily\fontsize{15.000000}{18.000000}\selectfont\catcode`\^=\active\def^{\ifmmode\sp\else\^{}\fi}\catcode`\%=\active\def
\end{pgfscope}%
\begin{pgfscope}%
\pgfsetroundcap%
\pgfsetroundjoin%
\pgfsetlinewidth{1.204500pt}%
\definecolor{currentstroke}{rgb}{0.000000,0.000000,0.000000}%
\pgfsetstrokecolor{currentstroke}%
\pgfsetdash{}{0pt}%
\pgfpathmoveto{\pgfqpoint{0.135642in}{2.072966in}}%
\pgfpathlineto{\pgfqpoint{0.343975in}{2.072966in}}%
\pgfpathlineto{\pgfqpoint{0.552309in}{2.072966in}}%
\pgfusepath{stroke}%
\end{pgfscope}%
\begin{pgfscope}%
\definecolor{textcolor}{rgb}{0.150000,0.150000,0.150000}%
\pgfsetstrokecolor{textcolor}%
\pgfsetfillcolor{textcolor}%
\pgftext[x=0.718975in,y=2.000049in,left,base]{\color{textcolor}{\sffamily\fontsize{15.000000}{18.000000}\selectfont\catcode`\^=\active\def^{\ifmmode\sp\else\^{}\fi}\catcode`\%=\active\def
\end{pgfscope}%
\begin{pgfscope}%
\pgfsetbuttcap%
\pgfsetroundjoin%
\pgfsetlinewidth{1.204500pt}%
\definecolor{currentstroke}{rgb}{0.000000,0.000000,0.000000}%
\pgfsetstrokecolor{currentstroke}%
\pgfsetdash{{1.200000pt}{1.980000pt}}{0.000000pt}%
\pgfpathmoveto{\pgfqpoint{0.135642in}{1.767180in}}%
\pgfpathlineto{\pgfqpoint{0.343975in}{1.767180in}}%
\pgfpathlineto{\pgfqpoint{0.552309in}{1.767180in}}%
\pgfusepath{stroke}%
\end{pgfscope}%
\begin{pgfscope}%
\definecolor{textcolor}{rgb}{0.150000,0.150000,0.150000}%
\pgfsetstrokecolor{textcolor}%
\pgfsetfillcolor{textcolor}%
\pgftext[x=0.718975in,y=1.694263in,left,base]{\color{textcolor}{\sffamily\fontsize{15.000000}{18.000000}\selectfont\catcode`\^=\active\def^{\ifmmode\sp\else\^{}\fi}\catcode`\%=\active\def
\end{pgfscope}%
\begin{pgfscope}%
\definecolor{textcolor}{rgb}{0.150000,0.150000,0.150000}%
\pgfsetstrokecolor{textcolor}%
\pgfsetfillcolor{textcolor}%
\pgftext[x=0.718975in,y=1.082691in,left,base]{\color{textcolor}{\sffamily\fontsize{15.000000}{18.000000}\selectfont\catcode`\^=\active\def^{\ifmmode\sp\else\^{}\fi}\catcode`\%=\active\def
\end{pgfscope}%
\begin{pgfscope}%
\pgfsetbuttcap%
\pgfsetroundjoin%
\pgfsetlinewidth{2.007500pt}%
\definecolor{currentstroke}{rgb}{0.000000,0.000000,1.000000}%
\pgfsetstrokecolor{currentstroke}%
\pgfsetdash{{7.400000pt}{3.200000pt}}{0.000000pt}%
\pgfpathmoveto{\pgfqpoint{0.135642in}{0.849822in}}%
\pgfpathlineto{\pgfqpoint{0.343975in}{0.849822in}}%
\pgfpathlineto{\pgfqpoint{0.552309in}{0.849822in}}%
\pgfusepath{stroke}%
\end{pgfscope}%
\begin{pgfscope}%
\definecolor{textcolor}{rgb}{0.150000,0.150000,0.150000}%
\pgfsetstrokecolor{textcolor}%
\pgfsetfillcolor{textcolor}%
\pgftext[x=0.718975in,y=0.776906in,left,base]{\color{textcolor}{\sffamily\fontsize{15.000000}{18.000000}\selectfont\catcode`\^=\active\def^{\ifmmode\sp\else\^{}\fi}\catcode`\%=\active\def
\end{pgfscope}%
\begin{pgfscope}%
\pgfsetbuttcap%
\pgfsetroundjoin%
\pgfsetlinewidth{2.007500pt}%
\definecolor{currentstroke}{rgb}{0.839216,0.643137,0.000000}%
\pgfsetstrokecolor{currentstroke}%
\pgfsetdash{{2.000000pt}{3.300000pt}}{0.000000pt}%
\pgfpathmoveto{\pgfqpoint{0.135642in}{0.544036in}}%
\pgfpathlineto{\pgfqpoint{0.343975in}{0.544036in}}%
\pgfpathlineto{\pgfqpoint{0.552309in}{0.544036in}}%
\pgfusepath{stroke}%
\end{pgfscope}%
\begin{pgfscope}%
\definecolor{textcolor}{rgb}{0.150000,0.150000,0.150000}%
\pgfsetstrokecolor{textcolor}%
\pgfsetfillcolor{textcolor}%
\pgftext[x=0.718975in,y=0.471120in,left,base]{\color{textcolor}{\sffamily\fontsize{15.000000}{18.000000}\selectfont\catcode`\^=\active\def^{\ifmmode\sp\else\^{}\fi}\catcode`\%=\active\def
\end{pgfscope}%
\begin{pgfscope}%
\pgfsetbuttcap%
\pgfsetroundjoin%
\pgfsetlinewidth{2.007500pt}%
\definecolor{currentstroke}{rgb}{0.000000,0.000000,0.000000}%
\pgfsetstrokecolor{currentstroke}%
\pgfsetdash{{12.800000pt}{3.200000pt}{2.000000pt}{3.200000pt}}{0.000000pt}%
\pgfpathmoveto{\pgfqpoint{0.135642in}{0.238251in}}%
\pgfpathlineto{\pgfqpoint{0.343975in}{0.238251in}}%
\pgfpathlineto{\pgfqpoint{0.552309in}{0.238251in}}%
\pgfusepath{stroke}%
\end{pgfscope}%
\begin{pgfscope}%
\definecolor{textcolor}{rgb}{0.150000,0.150000,0.150000}%
\pgfsetstrokecolor{textcolor}%
\pgfsetfillcolor{textcolor}%
\pgftext[x=0.718975in,y=0.165334in,left,base]{\color{textcolor}{\sffamily\fontsize{15.000000}{18.000000}\selectfont\catcode`\^=\active\def^{\ifmmode\sp\else\^{}\fi}\catcode`\%=\active\def
\end{pgfscope}%
\end{pgfpicture}%
\makeatother%
\endgroup%

%% file: sections/05_discussion_conclusion.tex
\subsection{Discussion}
The results demonstrate that probabilistic shielding can effectively enhance offline SPI methods. 
By constraining learned policies to avoid actions that violate reach-avoid specifications with high probability, our approach yields better average and worst-case performance, particularly in settings where data is scarce.
These findings provide positive answers to \ref{q1} and \ref{q2}.
Additionally, the results also show that the shielded SPI methods learn safer policies with fewer unsafe outcomes, positively answering \ref{q3}.

The performance of the shield depends on the quality of the estimated IMDP~$\hat\imdp$ learned from data.
In cases where the dataset is too sparse or the interval estimates are poorly calibrated, the transition intervals might be too wide or inaccurate to be useful in shielding.
Since the shield utilizes probabilities computed from the worst-case transition function within the IMDP, wide intervals may result in an overly conservative shield.
On the other hand, intervals that are too narrow may not capture the actual transition dynamics and therefore fail to identify genuine risks.

In some benchmarks, the shielded baseline performance declines for larger datasets.
This happens because we use a fixed safety threshold~$\threshold$ for all dataset sizes.
As the dataset size increases, the estimated intervals for each transition typically get smaller.
Because we use the worst-case transitions within the uncertainty set, a larger dataset generally results in a lower probability of violating the reach-avoid specification. 
Thus, more actions may be found to be $\threshold$-safe, and the shielded baseline incurs lower performance.
It is challenging to find a threshold value that works for both small datasets with more conservative probabilities and larger datasets with more accurate probabilities.
Since the shield is most effective for smaller datasets, we selected a threshold value that works well for these dataset sizes, and kept it consistent across all dataset sizes.

Finally, we note that a limitation of the shield's design is that there may be a trade-off between safety and optimality.
By construction, the shield restricts the action space based on worst-case reach-avoid probabilities. 
By constraining action choices based on worst-case estimates, the shield may block actions that offer high rewards but also entail a small risk of failure.
This conservative behavior can lead to suboptimal policies, particularly when the safety threshold and confidence parameters are poorly calibrated.

\section{Related work}
Here, we provide an overview of related literature in the fields of (probabilistic) shielding (in online RL) and SPI (in offline RL).
\paragraph{Shielding}
\citet{DBLP:conf/aaai/AlshiekhBEKNT18} 
synthesize shields by constructing a \emph{safety game} that combines a \emph{linear temporal logic} specification with an abstract model of the environment. 
Solving the game prevents unsafe actions while allowing the agent to learn optimally. 
\citet{DBLP:conf/concur/0001KJSB20} 
calculate probabilistic shields, shielding actions that are relative to the optimal safety. 
\citet{DBLP:conf/aaai/Carr0JT23} calculate reach-avoid shields, yet in the context of deep RL and under partial observability. 
\citet{DBLP:conf/amcc/PrangerKTD0B21}
synthesize a shield adaptively, using separate cost functions for performance and interference, and calculate a shield based on a weighted cost objective.
\citet{court_probabilistic_2025} guarantee probabilistic safety by constructing a continuous MDP that embeds the probabilistic safety constraints and apply deep reinforcement learning to find a safe policy.
Notable extensions of shielding include hybrid~\citep{DBLP:conf/vecos/BrorholtJLLS23,DBLP:conf/vecos/BrorholtHLS24} and multi-agent systems~\citep{DBLP:conf/ifaamas/BrorholtL025,DBLP:conf/icml/0002ZCSSF24,DBLP:conf/atal/MelcerAT24}.
All the above assume some form of a \emph{safety-relevant} model, where the part of the MDP relevant to safety is assumed to be fully known.
In contrast, \citet{DBLP:conf/isola/TapplerPKMBL22} compute shields for exploration of an unknown MDP using automata learning.
Furthermore, in model-free (deep) RL,
\citet{DBLP:conf/aaai/HeLB22} introduce \emph{latent shielding}, simulating future trajectories in a latent space and estimating the likelihood of an action leading to a safety violation.
\citet{DBLP:conf/ecai/GoodallB23} build on this idea and introduce \emph{approximate model-based shielding}, 
to look-ahead for safety.
Yet, all the above literature focuses on shielding in online scenarios, where the shields limit the states the agent can visit during exploration.  
Instead, we apply shields to constrain the space of policies that can be learned in offline RL.

\paragraph{Safe Policy Improvement}
In SPI~\citep{DBLP:conf/icml/PirottaRPC13} for offline RL, prior work has focused on improving an (assumed to be safe) baseline policy with high confidence. 
HCPI~\citep{DBLP:conf/aaai/ThomasTG15} uses a high-confidence lower bound on the expected return of a candidate policy using importance sampling estimates.
~\citet{DBLP:conf/nips/GhavamzadehPC16} 
formalize and provide approximations to the robust policy improvement problem, also introducing a method that 
penalizes the reward function based on how often a state-action pair occurs in the dataset.
\citet{DBLP:conf/nips/ChandakJTWT20} study SPI in non-stationary MDPs.
In the context of SPIBB,
\citet{DBLP:conf/atal/SimaoLC20} estimate the baseline policy from data, and \citet{DBLP:conf/nips/SatijaTPL21} study multiple objectives, incorporating constraints on expected costs. 
Furthermore, \citet{DBLP:conf/ijcai/WienhoftSSDB023} address the sample complexity of SPI, \citet{DBLP:journals/corr/abs-2507-15532} exploit parametric and graphical structure and \citet{DBLP:conf/icml/Castellini0ZSFS23} scale up using Monte Carlo tree search.
\citet{DBLP:conf/pkdd/NadjahiLC19} introduce Soft-SPIBB, relaxing constraints on policy search, and
\citet{DBLP:conf/icaart/SchollDOU22a} identify and propose solutions for limitations of Soft-SPIBB.

\section{Conclusion}
We presented a novel method for integrating probabilistic shielding into offline RL within the SPI framework. 
We construct a shield from offline data and enforce it onto the baseline policy and during policy improvement, such that
with high probability, actions that are likely to lead to unsafe outcomes are disabled.
Thereby, we circumvent assumptions regarding (1) the safety of the baseline policy that is common in the SPI literature, and (2) the availability of a sufficient safety-relevant model, which is common in the shielding literature.
Empirical results demonstrate that shielded SPI methods improve both average and worst-case performance across the four benchmark environments.
Moreover, they consistently outperform their non-shielded counterparts when data is scarce. 
In future work, we will bridge our computation of the shield from a dataset to the online RL setting.

%% file: sections/A1-complete_graphs.tex
\onecolumn

\section{Estimating the baseline policy}
\label{sec:est_baseline}
When using SPI algorithms, it is typically assumed that the baseline policy is given.
However, \citet{DBLP:conf/atal/SimaoLC20} previously demonstrated that this assumption can be relaxed for SPIBB, allowing an estimated baseline to be used in its place.
The estimation is done by using the occurrences of states $N_\dataset(s)$ and state-action pairs $N_\dataset(s,a)$ from the given dataset of trajectories $\dataset$.
For each state, we normalize these counts to obtain a probability distribution over actions. 
If a state has not been visited in the dataset, we default to a uniform distribution over actions.
The estimated baseline policy $\hat{\pi}_b$ is:
\begin{equation}
    \hat{\pi}^b(a \mid s) = 
    \begin{cases}
    \nicefrac{N_{\dataset}(s,a)}{N_{\dataset}(s)} & \text{if } N_{\dataset}(s) > 0, \\
    \nicefrac{1}{|A|} & \text{otherwise.}
    \end{cases}
\end{equation}
In the paper, we denote the baseline policy by $\baseline$ as if it were known.
Yet, we do not make this assumption in practice and instead use the estimated baseline $\hat{\pi}^b \approx \baseline$ in our experiments.

\section{Point estimates}
\label{app:map}
PAC learning is a standard method for estimating transition functions of MDPs. 
The process begins by obtaining a point estimate $\hat{T}$, e.g., via maximum likelihood or otherwise.
In our implementation, we use maximum a posteriori probability (MAP) estimates. 
Suppose we have the state-action pair $(s,a)$ with $m$ successor states, and each successor state $s'_i$ is observed $k_i$ times. 
Without loss of generality, for brevity we assume here that each $(s,a)$ pair has $m$ successor states.
We use the Dirichlet distribution 
$\dir(\T(\cdot\mid s, a) \mid \alpha_1, \dots, \alpha_m) \propto \prod_{i=1}^{m} \T(s'_i \mid s, a)^{\alpha_i - 1}$ 
as a conjugate prior to the multinomial likelihood. 
Because it is a conjugate prior, we do not need to compute the posterior distribution explicitly.
Instead, it is sufficient to update the parameters of the prior Dirichlet distribution. 
This way, we get the posterior distribution $\dir(\T(\cdot \mid s, a) \mid \alpha_1 + k_1, \dots, \alpha_m + k_m)$, where $k_i$ is the number of times successor state $i \in \{1, \ldots , m\}$ has been observed. 
The MAP estimate can now be found by taking the mode:

\begin{equation}
    \tilde{\T}(s'_i\mid s, a) = \frac{\alpha_i + k_i- 1}{\sum_{j=1}^{m}\left(\alpha_j + k_j \right) - m}
\end{equation}
This estimated transition function is likely to differ from the actual transition function. 
Yet the point estimate does not account for the uncertainty in the estimate.
Therefore, as is common, we turn this point estimate into a confidence interval using Hoeffding’s inequality~\cite{Hoeffding}. 

\section{Complete Plots}
\label{app:plots}
For some of the plots shown in \cref{fig:results}, the y-axis values have been clipped for better visualization:  
\begin{itemize}
    \item The Wet Chicken y-axis is clipped at -60,  
    \item The Frozen Lake y-axis is clipped at -5, and  
    \item The Random MDPs y-axis is clipped at -4.  
\end{itemize}

\begin{table}[tb]
    \caption{The dimensions of the benchmarks and full range of hyperparameters used in the experiments.}
       \centering
    \begin{tabular}{lrrrrrrrrrr}
    \toprule
    Benchmarks  & $|S|$ & $|A|$ & $\nwedge$ & $\uncertaintyweight$ & $\epsilon$ & $\theta$         & $\kappa$         & $\errorrate$      & $\alpha$ & $\probmargin$\\ \midrule
    \textit{Random MDPs}                 & 50 & 4 &3         & 0.1                  & 0.5        & 0.2              & 0.05             & 0.1               & 5       & $1 \times 10^{-8}$   \\
    \textit{Wet Chicken}         & 26 & 5         & 7         & 0.05                 & 0.05        & 0.2              & 0.5              & 0.1               & 2       & $1 \times 10^{-8}$    \\
    \textit{Frozen Lake}            & 64 & 4     & 3         & 1                    & 0.5        & 0.2              & 0.02             & 0.1               & 5     & $1 \times 10^{-8}$      \\ 
    \textit{Pacman}              & 118k & 4        & 3         & 1                    & 0.5        & 0.01             & 0.01             & 0.1               & 10      & $1 \times 10^{-8}$    \\
    \bottomrule
    \end{tabular}
        \label{app:tab:parameters}
\end{table}

We opted for clipping because the CVaR performance was very low in these cases, which would have otherwise made the plots difficult to read and interpret.

Included in this appendix are two figures that plot the complete data. 
\cref{fig:results_avg} shows the average performance with error bars representing 95\% confidence intervals. \cref{fig:results_cvar} shows the unclipped 1\%-CVaR data.

\begin{figure}[tbp]
    \providecommand\mathdefault[1]{#1}
    \everymath=\expandafter{\the\everymath\displaystyle}
    \renewcommand\sffamily{}
\resizebox{\textwidth}{!}{\input{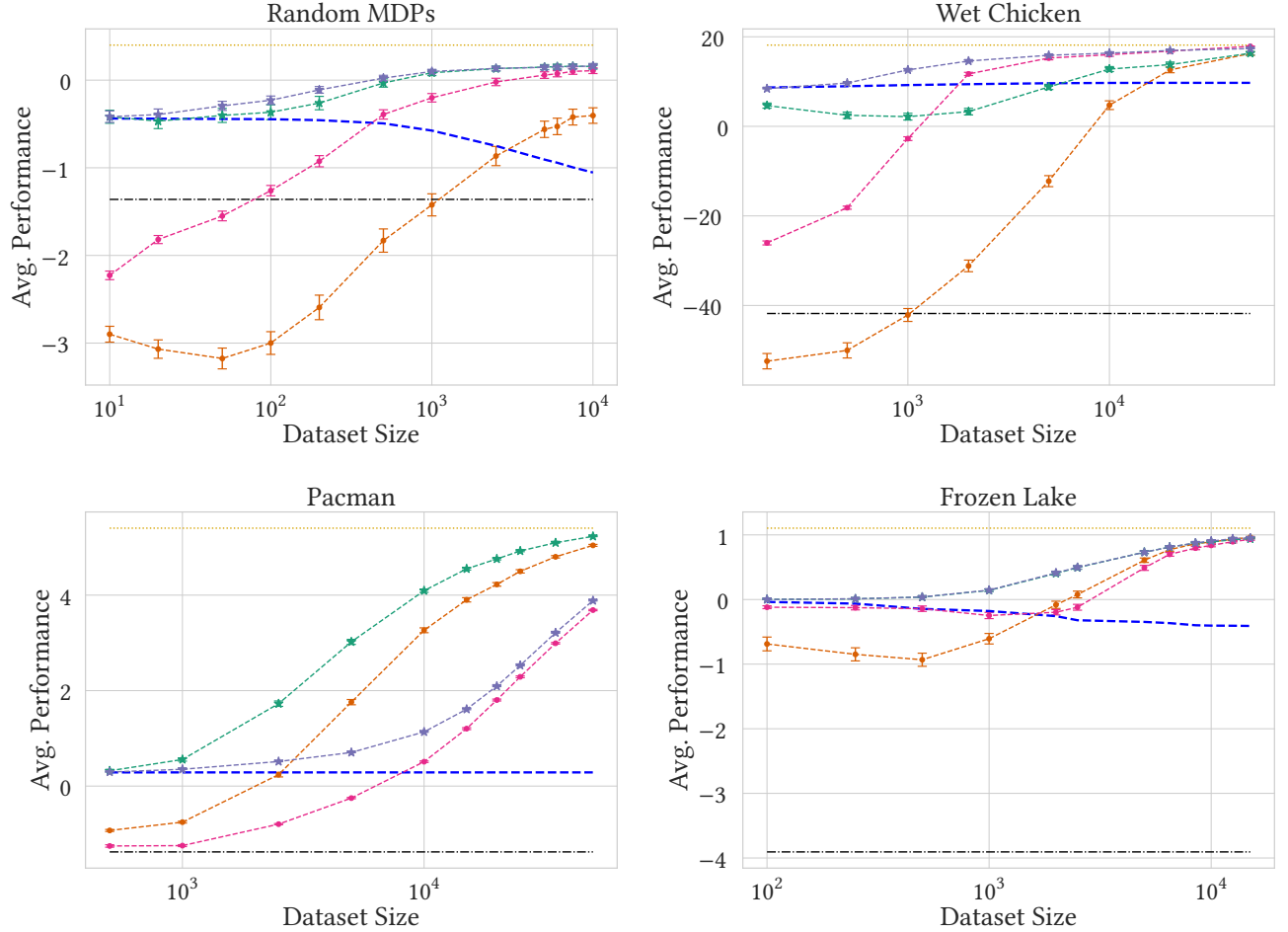}}
\caption{Full-range plots of the average performance of the shielded and non-shielded methods plotted against the number of trajectories in the dataset~$\dataset$ (log-scale), including error bars representing 95\% confidence intervals.}\label{fig:results_avg}
\end{figure}
\begin{figure}[bp]
    \providecommand\mathdefault[1]{#1}
    \everymath=\expandafter{\the\everymath\displaystyle}
    \renewcommand\sffamily{}
\resizebox{\textwidth}{!}{\input{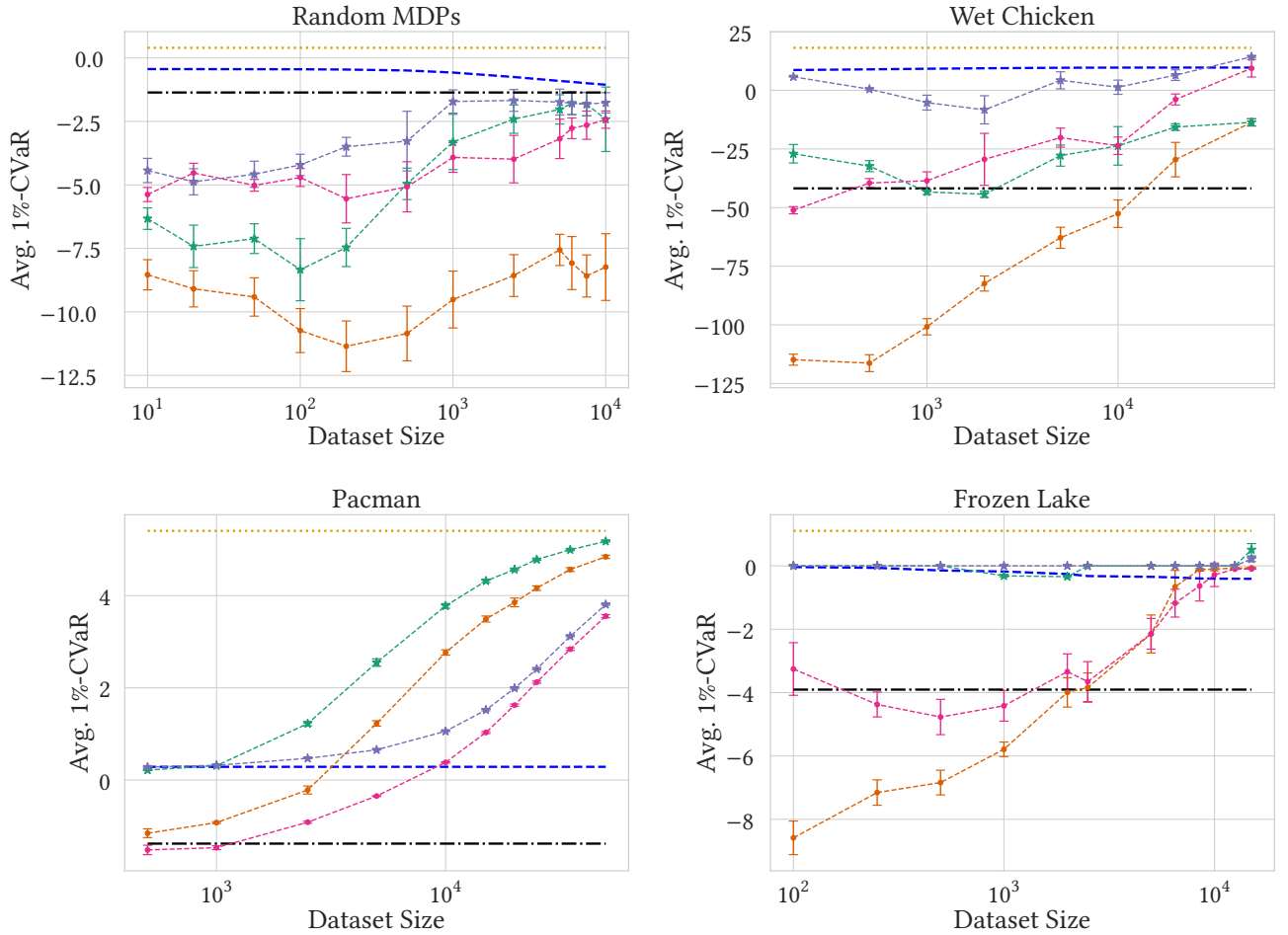}}
\caption{Full-range plots of the 1\%-CVaR performance of the shielded and non-shielded methods plotted against the number of trajectories in the dataset~$\dataset$ (log-scale). Unlike the clipped versions shown in the main text, these plots display the entire data range}\label{fig:results_cvar}

\end{figure}

\section{Benchmarks}
\label{app:benchmarks}
We use four different environments to benchmark the performance of our approach. 
Each of these benchmarks, along with its baseline policy and safety specifications, will be explained in this subsection.
The baseline policies are estimated from data, as explained in \Cref{sec:est_baseline}.

\paragraph{Random MDPs}
Our first experimental environment is a slightly modified version of the Random MDPs benchmark, as introduced in Nadjah et al.~\cite{DBLP:conf/pkdd/NadjahiLC19}. 
It consists of a randomly generated MDP with $|S| = 50$ states.
Each state offers four actions, which can lead to four different states, with probabilities determined randomly. 
The agent starts in $s_0$, and the goal state is selected to minimize the optimal value function.
The environment also contains five randomly selected sink states, or ``traps''.
These trapped states are chosen such that all remain reachable from the initial state. 
This modification was necessary as the original benchmark lacked unsafe states.
All transitions yield a reward of 0 except for those leading into the final state or a trap state. 
A new random MDP is generated for each iteration.

The safety specification is defined as 
$\varphi = (\neg \text{trap } \mathsf{U} \text{ goal})$, requiring the agent to reach the goal while avoiding trap states.
Here, the heuristic policy $\heuristic$ used to construct the baseline policy is the optimal policy $\pi^* \in \Pi$.

\paragraph{Wet Chicken}
The Wet Chicken benchmark~\cite{DBLP:conf/icann/HansU09,DBLP:conf/icaart/SchollDOU22a} models an agent controlling a boat on a river adjacent to a waterfall.
We follow the description of \citet{DBLP:conf/icaart/SchollDOU22a}.
The objective is to stay as close to the edge of the waterfall without falling off. 
This river is represented by a 5 $\times$ 5 grid world. 
If the x position of the boat ever exceeds 5, it falls over the waterfall, its position resets to $(0,0)$, and a large negative reward is obtained.
Otherwise, the reward obtained in each subsequent timestep is equal to the $x$-position of the boat. 
The agent can perform a set of predefined actions: it can drift with no movement, corresponding to an action vector of $(a_x, a_y) = (0, 0)$, hold, which involves slightly paddling away from the waterfall with $(a_x, a_y) = (-1, 0)$, paddle back, representing a stronger motion away from the waterfall, with $(a_x, a_y) = (-2, 0)$, steer right, with $(a_x, a_y) = (0, 1)$, or steer left, with $(a_x, a_y) = (0, -1)$.
The boat’s movement is also influenced by the stream $v_t = y_t \frac{3}{5}$ and turbulence $b_t = 3.5 - v_t$, where the turbulence direction $\tau_t \in [-1,1]$ is sampled uniformly at random each timestep.
The resulting dynamics are~$(x_{t+1}, y_{t+1}) = \text{round}(x_t+a_x+v_t+\tau_tb_t), \text{round}(x_t+a_y)$.
The safety specification is $\varphi =(\neg \text{waterfall } \mathsf{U} \text{ edge})$, requiring the agent to reach the states adjacent to the waterfall without falling.

\paragraph{Frozen Lake}
We employ a modified version of the Frozen Lake environment from the Gymnasium library~\cite{towers2024gymnasiumstandardinterfacereinforcement}.
The task involves navigating an $8 \times 8$ grid to reach a goal state while avoiding hazardous hole (sink) states. 
Reaching such a state results in a negative reward, and reaching the goal gives a positive reward.
The agent can move in the four cardinal directions, but due to the slippery surface, there is a chance of deviating from the intended direction at a perpendicular angle.

The safety specification is defined as $\varphi=(\neg \text{hole } \mathsf{U} \text{ goal})$, requiring the agent to reach the goal without first falling into a hole.

The heuristic baseline policy $\pi_H$ selects downward and rightward actions with equal probability.
The final baseline policy $\pi_b$ is formed by mixing $\pi_H$ with the uniformly random policy $\tilde{\pi}$, weighted by parameter $\epsilon$.

\paragraph{Simplified Pacman}
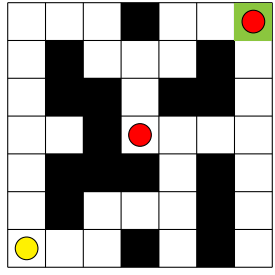
\begin{figure}
    \centering
    \begin{tikzpicture}[scale=0.5]
        \foreach \x in {0,...,7} {
            \draw (\x,0) -- (\x,7);
        }
        \foreach \y in {0,...,7} {
            \draw (0,\y) -- (7,\y);
        }

        \foreach \x/\y in {
            1/1, 1/2, 1/4, 1/5,
            2/2, 2/3, 2/4,
            3/0, 3/2, 3/6,
            4/4,
            5/0, 5/1, 5/2, 5/4, 5/5
        } {
            \fill[black] (\x,\y) rectangle ++(1,1);
        }
        \fill[LimeGreen] (6,6) rectangle ++(1,1);
        \fill[yellow] (0.5,0.5) circle (0.3);
        \draw (0.5,0.5) circle (0.3);

        \fill[red] (6.5,6.5) circle (0.3);  
        \draw (6.5,6.5) circle (0.3);

        \fill[red] (3.5,3.5) circle (0.3);  
        \draw (3.5,3.5) circle (0.3);

    \end{tikzpicture}
    \caption{Starting position of the pacman environment}
    \label{fig:pacman7x7}
\end{figure}

We use an environment that models a simplified Pacman game, inspired by Alshiekh et al.~\cite{DBLP:conf/aaai/AlshiekhBEKNT18}. 
The environment is a $7 \times 7$ grid maze where the agent starts in the bottom-left corner and aims to reach the top-right goal state. 
The agent can move in four cardinal directions: up, right, down, and left. 
Two ghost agents also move randomly throughout the maze.
If the agent attempts to move into a wall, its position is reset to the initial state. 
If the agent collides with a wall, it is reset to the initial state. 
A collision with a ghost causes the agent to be eaten, resulting in a negative reward and terminating the episode.
See \cref{fig:pacman7x7} for a visual representation of the environment.

The state space size is $(7^2)^3 = 117649$, as the positions of the agent and two ghosts must be tracked.
Using this environment, we demonstrate the scalability of our approach to large state spaces.
The safety specification is $\varphi =(\neg \text{eaten } \mathsf{U} \text{ goal})$, requiring the agent to reach the goal while avoiding ghosts.
The heuristic baseline policy $\pi_H$ selects upward and rightward actions with equal probability, and disallows actions that would result in collisions with walls.
The final baseline policy $\pi_b$ is a mixture of $\pi_H$ and the uniformly random policy $\tilde{\pi}$, weighted by parameter $\epsilon$.

\input{sections/duipi}
\label{app:duipi}

Rather than using a full covariance-based formulation, which is computationally expensive for large state spaces, DUIPI employs an approximate algorithm that utilizes only the diagonal of the covariance matrix. 
We assume independence between components.
\citet{schneegassDuipi} found that this approximation can be made for larger MDPs with more than 18 states and 4 actions, which is the case for all environments used in this paper.
When neglecting correlations, we can estimate the variance of a transformation $f$ on vector $X$ using~\cite{schneegassDuipi,DBLP:conf/icaart/SchollDOU22a}:

\begin{equation}
    \var(f(X)) \approx \sum_i \left( \frac{\partial f}{\partial X_i} \right)^2 \var(X_i) 
\end{equation}

When estimating the uncertainty of the action-value function $Q^{\pi}_{\hat{\mdp}}$, we must account for the variance introduced by $\T$ and $R$, as well as the variance propagated through $V^{\pi}_{\mdp}$.
We can therefore estimate the variance of $Q^{\pi}_{\hat{\mdp}}$ using~\cite{schneegassDuipi,DBLP:conf/icaart/SchollDOU22a}:

\begin{align}
\operatorname{Var}(Q^{\pi}_{\hat{\mdp}}(s, a)) &\approx\ 
 \sum_{s'} \gamma^2 \hat{\T}(s' \mid s,a)^2 \operatorname{Var}(V^{\pi}_{\hat{\mdp}}(s')) \notag \\
 &+ \sum_{s'} \left( \hat{R}(s,a,s') + \gamma V^{\pi}_{\hat{\mdp}}(s') \right)^2 \operatorname{Var}(\hat{\T}(s' \mid s,a))  \\
 &+ \sum_{s'} \hat{\T}(s' \mid s,a)^2 \operatorname{Var}(\hat{R}(s,a,s'))\notag
\end{align}
Each term corresponds to a distinct source of uncertainty.

%% file: sections/duipi.tex
\subsection{DUIPI}
\label{subsec:DUIPI}
The \emph{diagonal approximation of uncertainty incorporating policy iteration} (DUIPI)~\cite{schneegassDuipi} algorithm penalizes the reward of state-action pairs with high uncertainty by calculating the variance. 

In DUIPI, the variance of the action value $\var(Q^{\pi}_{\hat{M}}(s,a))$ is computed by considering the variance of the estimated transition function $\hat{\T}$.
In the PE step, this variance penalizes the expected reward:
    \(
    U(s,a) = Q^{\pi}_{\hat{M}}(s,a) -  \uncertaintyweight \sqrt{\var(Q^{\pi}_{\hat{M}}(s,a))},
    \)
where $\uncertaintyweight$ is a hyperparameter that determines the degree to which the variance influences the uncertainty.
We provide details on the computation of $U$ and $\hat{\T}$ in \cref{app:duipi}.
The resulting probabilistic guarantee is that with probability $1 - F(\zeta)$, the true action-value $Q^\pi_M$ exceeds the value $U$:
\begin{equation}
    \Pr(Q_M^\pi(s,a) > U(s,a)) = 1 - F(\zeta),
\end{equation}
where $F(\zeta) = \Pr(Z \leq \zeta)$ is the cumulative density function for a standard normal distributed $Z$.
Here, it is assumed that the action-value function is normally distributed, and that covariance between state-action pairs is neglible.
We refer to~\citet{schneegassDuipi} and \citet{DBLP:conf/icaart/SchollDOU22a} for a more extensive discussion.

During policy iteration, DUIPI updates the baseline policy by iteratively increasing the probability of the action with the highest penalized action-value~$U$.
Actions not selected as best are adjusted accordingly to maintain a valid probability distribution.
The rate of change decreases as the iteration count increases.
Formally, at policy iteration step $t\in\mathbb{N}$, we update the policy $\pi'$, using $a^*_{U}(s) = \argmax_{a\in A}U(s,a)$, as:
\begin{equation}
    \pi'(a\mid s)=
    \begin{cases}
        \min\left(\pi(a\mid s) + \nicefrac{1}{t}, 1\right) & \text{if } a \in a^*_{U}(s),\\[8pt]
        \dfrac{\max\left(1-\pi(a^*_{U}(s) \mid s)-\nicefrac{1}{t}, 0\right)}{1-\pi(a^*_{U}(s) \mid s)} & \text{otherwise.}
    \end{cases}
\end{equation}
As it proved more performant, we opted for the Bayesian version of DUIPI, estimating $\hat{\T}$ with Dirichlets, and, instead of randomly initializing the initial policy used in the policy iteration procedure, we used the baseline policy $\baseline$ as an initial guess.

\paragraph{Shielded-DUIPI} 
For shielded-DUIPI, we amend the action-value function by assigning the minimum reward to any unsafe state-action pair.
Recall $U$ from \cref{subsec:DUIPI}.
Then, this results in the shielded value $Q_{\shield}$, with:
\begin{equation}
    U_{\shield}(s, a) =
    \begin{cases}
     U(s,a) & \text{if } (s,a) \in SA_{\shield}, \\[10pt]
    -\infty
    & \text{otherwise.}
    \end{cases}  
\end{equation}
Thus, unsafe actions cannot be selected.
In this case, we use the shielded baseline policy $\shieldedbaseline$ as the initial guess instead of the original baseline policy $\baseline$.

%% file: references.bib
@article{DBLP:journals/corr/abs-2005-01643,
  author       = {Sergey Levine and
                  Aviral Kumar and
                  George Tucker and
                  Justin Fu},
  title        = {Offline Reinforcement Learning: Tutorial, Review, and Perspectives
                  on Open Problems},
  journal      = {CoRR},
  volume       = {abs/2005.01643},
  year         = {2020}
}

@inproceedings{DBLP:conf/aaai/HeLB22,
  author       = {Chloe He and
                  Borja G. Le{\'{o}}n and
                  Francesco Belardinelli},
  title        = {Do Androids Dream of Electric Fences? Safety-Aware Reinforcement Learning
                  with Latent Shielding},
  booktitle    = {SafeAI@AAAI},
  series       = {{CEUR} Workshop Proceedings},
  volume       = {3087},
  publisher    = {CEUR-WS.org},
  year         = {2022}
}

@inproceedings{DBLP:conf/ecai/GoodallB23,
  author       = {Alexander W. Goodall and
                  Francesco Belardinelli},
  title        = {Approximate Model-Based Shielding for Safe Reinforcement Learning},
  booktitle    = {{ECAI}},
  series       = {Frontiers in Artificial Intelligence and Applications},
  volume       = {372},
  pages        = {883--890},
  publisher    = {{IOS} Press},
  year         = {2023}
}

@inproceedings{DBLP:conf/ecml/ErnstGW03,
  author       = {Damien Ernst and
                  Pierre Geurts and
                  Louis Wehenkel},
  title        = {Iteratively Extending Time Horizon Reinforcement Learning},
  booktitle    = {{ECML}},
  series       = {Lecture Notes in Computer Science},
  volume       = {2837},
  pages        = {96--107},
  publisher    = {Springer},
  year         = {2003}
}

@inproceedings{DBLP:conf/icml/LarocheTC19,
  author       = {Romain Laroche and
                  Paul Trichelair and
                  Remi Tachet des Combes},
  title        = {Safe Policy Improvement with Baseline Bootstrapping},
  booktitle    = {{ICML}},
  series       = {Proceedings of Machine Learning Research},
  volume       = {97},
  pages        = {3652--3661},
  publisher    = {{PMLR}},
  year         = {2019}
}

@article{cacm:shields,
    author = {Bettina Koenighofer and Roderick Bloem and Nils Jansen and Sebastian Junges and Ste
fan Pranger},
    title = {Shields for safe reinforcement learning},
    journal = {CACM},
    year = {2025},
    notes = {(to appear)}
}

@inproceedings{DBLP:conf/vecos/BrorholtJLLS23,
  author       = {Asger Horn Brorholt and
                  Peter Gj{\o}l Jensen and
                  Kim Guldstrand Larsen and
                  Florian Lorber and
                  Christian Schilling},
  title        = {Shielded Reinforcement Learning for Hybrid Systems},
  booktitle    = {AISoLA},
  series       = {Lecture Notes in Computer Science},
  volume       = {14380},
  pages        = {33--54},
  publisher    = {Springer},
  year         = {2023}
}

@inproceedings{DBLP:conf/vecos/BrorholtHLS24,
  author       = {Asger Horn Brorholt and
                  Andreas Holck H{\o}eg{-}Petersen and
                  Kim Guldstrand Larsen and
                  Christian Schilling},
  title        = {Efficient Shield Synthesis via State-Space Transformation},
  booktitle    = {AISoLA},
  series       = {Lecture Notes in Computer Science},
  volume       = {15217},
  pages        = {206--224},
  publisher    = {Springer},
  year         = {2024}
}

@inproceedings{DBLP:conf/ifaamas/BrorholtL025,
  author       = {Asger Horn Brorholt and
                  Kim Guldstrand Larsen and
                  Christian Schilling},
  title        = {Compositional Shielding and Reinforcement Learning for Multi-Agent
                  Systems},
  booktitle    = {{AAMAS}},
  pages        = {399--407},
  publisher    = {International Foundation for Autonomous Agents and Multiagent Systems
                  / {ACM}},
  year         = {2025}
}

@inproceedings{DBLP:conf/icml/0002ZCSSF24,
  author       = {Federico Bianchi and
                  Edoardo Zorzi and
                  Alberto Castellini and
                  Thiago D. Sim{\~{a}}o and
                  Matthijs T. J. Spaan and
                  Alessandro Farinelli},
  title        = {Scalable Safe Policy Improvement for Factored Multi-Agent MDPs},
  booktitle    = {{ICML}},
  publisher    = {OpenReview.net},
  year         = {2024}
}

@inproceedings{DBLP:conf/atal/MelcerAT24,
  author       = {Daniel Melcer and
                  Christopher Amato and
                  Stavros Tripakis},
  title        = {Shield Decentralization for Safe Reinforcement Learning in General
                  Partially Observable Multi-Agent Environments},
  booktitle    = {{AAMAS}},
  pages        = {2384--2386},
  publisher    = {International Foundation for Autonomous Agents and Multiagent Systems
                  / {ACM}},
  year         = {2024}
}

@inproceedings{DBLP:conf/nips/SatijaTPL21,
  author       = {Harsh Satija and
                  Philip S. Thomas and
                  Joelle Pineau and
                  Romain Laroche},
  title        = {Multi-Objective {SPIBB:} Seldonian Offline Policy Improvement with
                  Safety Constraints in Finite MDPs},
  booktitle    = {NeurIPS},
  pages        = {2004--2017},
  year         = {2021}
}

@inproceedings{DBLP:conf/nips/ChandakJTWT20,
  author       = {Yash Chandak and
                  Scott M. Jordan and
                  Georgios Theocharous and
                  Martha White and
                  Philip S. Thomas},
  title        = {Towards Safe Policy Improvement for Non-Stationary MDPs},
  booktitle    = {NeurIPS},
  year         = {2020}
}

@inproceedings{DBLP:conf/isola/TapplerPKMBL22,
  author       = {Martin Tappler and
                  Stefan Pranger and
                  Bettina K{\"{o}}nighofer and
                  Edi Muskardin and
                  Roderick Bloem and
                  Kim G. Larsen},
  title        = {Automata Learning Meets Shielding},
  booktitle    = {ISoLA {(1)}},
  series       = {Lecture Notes in Computer Science},
  volume       = {13701},
  pages        = {335--359},
  publisher    = {Springer},
  year         = {2022}
}

@inproceedings{DBLP:conf/icml/Castellini0ZSFS23,
  author       = {Alberto Castellini and
                  Federico Bianchi and
                  Edoardo Zorzi and
                  Thiago D. Sim{\~{a}}o and
                  Alessandro Farinelli and
                  Matthijs T. J. Spaan},
  title        = {Scalable Safe Policy Improvement via Monte Carlo Tree Search},
  booktitle    = {{ICML}},
  series       = {Proceedings of Machine Learning Research},
  volume       = {202},
  pages        = {3732--3756},
  publisher    = {{PMLR}},
  year         = {2023}
}

@inproceedings{DBLP:conf/nips/GhavamzadehPC16,
  author       = {Mohammad Ghavamzadeh and
                  Marek Petrik and
                  Yinlam Chow},
  title        = {Safe Policy Improvement by Minimizing Robust Baseline Regret},
  booktitle    = {{NIPS}},
  pages        = {2298--2306},
  year         = {2016}
}

@article{DBLP:journals/corr/abs-2507-15532,
  author       = {Kasper Engelen and
                  Guillermo A. P{\'{e}}rez and
                  Marnix Suilen},
  title        = {Data-Efficient Safe Policy Improvement Using Parametric Structure},
  journal      = {CoRR},
  volume       = {abs/2507.15532},
  year         = {2025}
}

@inproceedings{DBLP:conf/icml/PirottaRPC13,
  author       = {Matteo Pirotta and
                  Marcello Restelli and
                  Alessio Pecorino and
                  Daniele Calandriello},
  title        = {Safe Policy Iteration},
  booktitle    = {{ICML} {(3)}},
  series       = {{JMLR} Workshop and Conference Proceedings},
  volume       = {28},
  pages        = {307--315},
  publisher    = {JMLR.org},
  year         = {2013}
}

@inbook{schneegassDuipi,
author = {Schneegass, Daniel and Hans, Alexander and Udluft, Steffen},
year = {2010},
month = {08},
pages = {},
title = {Uncertainty in Reinforcement Learning - Awareness, Quantisation, and Control},
isbn = {978-953-307-104-6},
doi = {10.5772/10250}
}

@inproceedings{DBLP:conf/icaart/SchollDOU22a,
  author       = {Philipp Scholl and
                  Felix Dietrich and
                  Clemens Otte and
                  Steffen Udluft},
  title        = {Safe Policy Improvement Approaches and Their Limitations},
  booktitle    = {{ICAART} (Revised Selected Paper},
  series       = {Lecture Notes in Computer Science},
  volume       = {13786},
  pages        = {74--98},
  publisher    = {Springer},
  year         = {2022}
}

@article{Hoeffding,
 ISSN = {01621459, 1537274X},
 URL = {http://www.jstor.org/stable/2282952},
 author = {Wassily Hoeffding},
 journal = {Journal of the American Statistical Association},
 number = {301},
 pages = {13--30},
 publisher = {[American Statistical Association, Taylor & Francis, Ltd.]},
 title = {Probability Inequalities for Sums of Bounded Random Variables},
 urldate = {2025-02-13},
 volume = {58},
 year = {1963}
}

@inproceedings{DBLP:conf/pkdd/NadjahiLC19,
  author       = {Kimia Nadjahi and
                  Romain Laroche and
                  R{\'{e}}mi Tachet des Combes},
  title        = {Safe Policy Improvement with Soft Baseline Bootstrapping},
  booktitle    = {{ECML/PKDD} {(3)}},
  series       = {Lecture Notes in Computer Science},
  volume       = {11908},
  pages        = {53--68},
  publisher    = {Springer},
  year         = {2019}
}

@inproceedings{DBLP:conf/icann/HansU09,
  author       = {Alexander Hans and
                  Steffen Udluft},
  title        = {Efficient Uncertainty Propagation for Reinforcement Learning with
                  Limited Data},
  booktitle    = {{ICANN} {(1)}},
  series       = {Lecture Notes in Computer Science},
  volume       = {5768},
  pages        = {70--79},
  publisher    = {Springer},
  year         = {2009}
}

@inproceedings{DBLP:conf/aaai/AlshiekhBEKNT18,
  author       = {Mohammed Alshiekh and
                  Roderick Bloem and
                  R{\"{u}}diger Ehlers and
                  Bettina K{\"{o}}nighofer and
                  Scott Niekum and
                  Ufuk Topcu},
  title        = {Safe Reinforcement Learning via Shielding},
  booktitle    = {{AAAI}},
  pages        = {2669--2678},
  publisher    = {{AAAI} Press},
  year         = {2018}
}

@misc{towers2024gymnasiumstandardinterfacereinforcement,
      title={Gymnasium: A Standard Interface for Reinforcement Learning Environments}, 
      author={Mark Towers and Ariel Kwiatkowski and Jordan Terry and John U. Balis and Gianluca De Cola and Tristan Deleu and Manuel Goulão and Andreas Kallinteris and Markus Krimmel and Arjun KG and Rodrigo Perez-Vicente and Andrea Pierré and Sander Schulhoff and Jun Jet Tai and Hannah Tan and Omar G. Younis},
      year={2024},
      eprint={2407.17032},
      archivePrefix={arXiv},
      primaryClass={cs.LG},
      url={https://arxiv.org/abs/2407.17032}, 
}

@article{Bellman1957AMD,
  title={A Markovian Decision Process},
  author={Richard Bellman},
  journal={Indiana University Mathematics Journal},
  year={1957},
  volume={6},
  pages={679-684},
  url={https://api.semanticscholar.org/CorpusID:123329493}
}

@inproceedings{MWW25,
  author       = {Tobias Meggendorfer and
                  Maximilian Weininger and
                  Patrick Wienh{\"{o}}ft},
  title        = {What Are the Odds? Improving the foundations of Statistical Model Checking},
  booktitle    = {{QEST + FORMATS}},
  year         = {2025 (to appear, preprint at \url{https://arxiv.org/abs/2404.05424})}
}

@inproceedings{DBLP:conf/ijcai/WienhoftSSDB023,
  author       = {Patrick Wienh{\"{o}}ft and
                  Marnix Suilen and
                  Thiago D. Sim{\~{a}}o and
                  Clemens Dubslaff and
                  Christel Baier and
                  Nils Jansen},
  title        = {More for Less: Safe Policy Improvement with Stronger Performance Guarantees},
  booktitle    = {{IJCAI}},
  pages        = {4406--4415},
  publisher    = {ijcai.org},
  year         = {2023}
}

@book{model_checking_bible,
  author       = {Christel Baier and
                  Joost{-}Pieter Katoen},
  title        = {Principles of model checking},
  publisher    = {{MIT} Press},
  year         = {2008}
}

@inproceedings{DBLP:conf/concur/0001KJSB20,
  author       = {Nils Jansen and
                  Bettina K{\"{o}}nighofer and
                  Sebastian Junges and
                  Alex Serban and
                  Roderick Bloem},
  title        = {Safe Reinforcement Learning Using Probabilistic Shields (Invited Paper)},
  booktitle    = {{CONCUR}},
  series       = {LIPIcs},
  volume       = {171},
  pages        = {3:1--3:16},
  publisher    = {Schloss Dagstuhl - Leibniz-Zentrum f{\"{u}}r Informatik},
  year         = {2020}
}

@inproceedings{DBLP:conf/aaai/Carr0JT23,
  author       = {Steven Carr and
                  Nils Jansen and
                  Sebastian Junges and
                  Ufuk Topcu},
  title        = {Safe Reinforcement Learning via Shielding under Partial Observability},
  booktitle    = {{AAAI}},
  pages        = {14748--14756},
  publisher    = {{AAAI} Press},
  year         = {2023}
}

@inproceedings{DBLP:conf/amcc/PrangerKTD0B21,
  author       = {Stefan Pranger and
                  Bettina K{\"{o}}nighofer and
                  Martin Tappler and
                  Martin Deixelberger and
                  Nils Jansen and
                  Roderick Bloem},
  title        = {Adaptive Shielding under Uncertainty},
  booktitle    = {{ACC}},
  pages        = {3467--3474},
  publisher    = {{IEEE}},
  year         = {2021}
}

@inproceedings{DBLP:conf/aaai/ThomasTG15,
  author       = {Philip S. Thomas and
                  Georgios Theocharous and
                  Mohammad Ghavamzadeh},
  title        = {High-Confidence Off-Policy Evaluation},
  booktitle    = {{AAAI}},
  pages        = {3000--3006},
  publisher    = {{AAAI} Press},
  year         = {2015}
}

@STRING{springer = {Springer}}

@STRING{acm = {ACM Press}}

@STRING{ieee = {IEEE}}

@STRING{lncs = {{Lecture Notes in Computer Science}}}

@STRING{cacm = {Communications of the ACM}}

@STRING{aaai = {{Proceedings of the AAAI Conference on Artificial Intelligence (AAAI)}}}

@STRING{aamas = {{International Conference on Autonomous Agents and Multiagent Systems (AAMAS)}}}

@STRING{acc = {{American Control Conference (ACC)}}}

@STRING{cav = {{Computer Aided Verification (CAV)}}}

@STRING{cdc = {CDC}}

@STRING{concur = {CONCUR}}

@STRING{ecml = {ECML}}

@STRING{formats = {FORMATS}}

@STRING{icml = {{International Conference on Machine Learning (ICML)}}}

@STRING{ijcai = {{International Joint Conference on Artificial Intelligence (IJCAI)}}}

@STRING{qest = {QEST}}

@STRING{tacas = {{Tools and Algorithms for the Construction and Analysis of Systems (TACAS)}}}

@STRING{nips = {{Advances in Neural Information Processing Systems (NIPS)}}}

@inproceedings{DBLP:conf/cav/DehnertJK017,
	author    = {Christian Dehnert and
	Sebastian Junges and
	Joost{-}Pieter Katoen and
	Matthias Volk},
	title     = {A Storm is Coming: {A} Modern Probabilistic Model Checker},
	booktitle = cav,
	series    = lncs,
	volume    = {10427},
	pages     = {592--600},
	publisher = {Springer},
	year      = {2017}
}

@inproceedings{KNP11,
	author    = {Marta Kwiatkowska and
	Gethin Norman and
	David Parker},
	title     = {\textsc{Prism} 4.0: Verification of Probabilistic Real-Time Systems},
	booktitle = cav,
	series    = lncs,
	publisher = springer,
	volume    = {6806},
	year      = {2011},
	pages     = {585--591}
}

@inproceedings{junges-et-al-tacas-2016,
	author    = {Sebastian Junges and
	Nils Jansen and
	Christian Dehnert and
	Ufuk Topcu and
	Joost{-}Pieter Katoen},
	title     = {Safety-Constrained Reinforcement Learning for {MDPs}},
	booktitle = tacas,
	series    = lncs,
	volume    = {9636},
	pages     = {130--146},
	publisher = springer,
	year      = {2016}
}

@inproceedings{draeger-et-al-tacas-2014,
	author    = {Klaus Dr{\"{a}}ger and
	Vojtech Forejt and
	Marta Z. Kwiatkowska and
	David Parker and
	Mateusz Ujma},
	title     = {Permissive Controller Synthesis for Probabilistic Systems},
	booktitle = tacas,
	series    = lncs,
	volume    = {8413},
	pages     = {531--546},
	publisher = springer,
	year      = {2014}
}

@inproceedings{DBLP:conf/cdc/WolffTM12,
  author    = {Eric M. Wolff and
               Ufuk Topcu and
               Richard M. Murray},
  title     = {Robust control of uncertain {M}arkov Decision Processes with temporal
               logic specifications},
  booktitle = {{CDC}},
  pages     = {3372--3379},
  publisher = {{IEEE}},
  year      = {2012}
}

@inproceedings{DBLP:conf/cav/PuggelliLSS13,
  author    = {Alberto Puggelli and
               Wenchao Li and
               Alberto L. Sangiovanni{-}Vincentelli and
               Sanjit A. Seshia},
  title     = {Polynomial-Time Verification of {PCTL} Properties of MDPs with Convex
               Uncertainties},
  booktitle = {{CAV}},
  series    = {Lecture Notes in Computer Science},
  volume    = {8044},
  pages     = {527--542},
  publisher = {Springer},
  year      = {2013}
}

@article{DBLP:journals/mor/WiesemannKR13,
  author    = {Wolfram Wiesemann and
               Daniel Kuhn and
               Ber{\c{c}} Rustem},
  title     = {Robust {M}arkov Decision Processes},
  journal   = {Math. Oper. Res.},
  volume    = {38},
  number    = {1},
  pages     = {153--183},
  year      = {2013}
}

@inproceedings{DBLP:conf/nips/SuilenS0022,
  author       = {Marnix Suilen and
                  Thiago D. Sim{\~{a}}o and
                  David Parker and
                  Nils Jansen},
  title        = {Robust Anytime Learning of {M}arkov Decision Processes},
  booktitle    = {NeurIPS},
  year         = {2022}
}

@article{DBLP:journals/mor/Iyengar05,
  author       = {Garud N. Iyengar},
  title        = {Robust Dynamic Programming},
  journal      = {Math. Oper. Res.},
  volume       = {30},
  number       = {2},
  pages        = {257--280},
  year         = {2005}
}

@article{DBLP:journals/ior/NilimG05,
  author       = {Arnab Nilim and
                  Laurent El Ghaoui},
  title        = {Robust Control of {M}arkov Decision Processes with Uncertain Transition
                  Matrices},
  journal      = {Oper. Res.},
  volume       = {53},
  number       = {5},
  pages        = {780--798},
  year         = {2005}
}

@inproceedings{DBLP:conf/atal/SimaoLC20,
  author       = {Thiago D. Sim{\~{a}}o and
                  Romain Laroche and
                  R{\'{e}}mi Tachet des Combes},
  title        = {Safe Policy Improvement with an Estimated Baseline Policy},
  booktitle    = {{AAMAS}},
  pages        = {1269--1277},
  publisher    = {International Foundation for Autonomous Agents and Multiagent Systems},
  year         = {2020}
}

@book{DBLP:books/wi/Puterman94,
  author       = {Martin L. Puterman},
  title        = {Markov Decision Processes: Discrete Stochastic Dynamic Programming},
  series       = {Wiley Series in Probability and Statistics},
  publisher    = {Wiley},
  year         = {1994}
}

@misc{court_probabilistic_2025,
    title = {Probabilistic {Shielding} for {Safe} {Reinforcement} {Learning}},
    url = {http://arxiv.org/abs/2503.07671},
    doi = {10.48550/arXiv.2503.07671},
    language = {en},
    urldate = {2025-09-26},
    publisher = {arXiv},
    author = {Court, Edwin Hamel-De le and Belardinelli, Francesco and Goodall, Alexander W.},
    month = mar,
    year = {2025},
    note = {arXiv:2503.07671 [stat]},
}
